\documentclass{article}
\usepackage[utf8]{inputenc}
\usepackage{times} 
\usepackage{helvet}
\usepackage{courier}
\usepackage[hyphens]{url}
\usepackage{graphicx}
\usepackage[numbers]{natbib}
\usepackage{caption}
\usepackage[letterpaper]{geometry}
\usepackage[ruled, linesnumbered]{algorithm2e}
\usepackage{calrsfs}
\DeclareMathAlphabet{\pazocal}{OMS}{zplm}{m}{n}
\usepackage{amsfonts}
\usepackage{amssymb}
\usepackage{amsthm}
\usepackage{bm}
\usepackage{booktabs}
\newtheorem{theorem}{Theorem}
\newtheorem{definition}{Definition}
\begin{document}

\title{Dataset Representativeness and Downstream Task Fairness}
\author{
    Victor Borza \textsuperscript{* \rm 1},
    Andrew Estornell \textsuperscript{* \rm 2,3},
    Chien-Ju Ho \textsuperscript{\rm 2},\\
    Bradley Malin \textsuperscript{\rm 1},
    Yevgeniy Vorobeychik \textsuperscript{\rm 2}
}
\date{}
\maketitle
\begin{center}
    \textsuperscript{\rm 1}Vanderbilt University, Nashville, TN, USA \\
    \textsuperscript{\rm 2}Washington University in St. Louis, St. Louis, MO, USA \\
    \textsuperscript{\rm 3}ByteDance Research, San Jose, CA, USA \\
\end{center}

\begin{center}
    \small
    * These authors contributed equally to this work.
\end{center}

\begin{abstract}
Our society collects data on people for a wide range of applications, from building a census for policy evaluation to running meaningful clinical trials. 
To collect data, we typically sample individuals with the goal of accurately representing a population of interest. 
However, current sampling processes often collect data opportunistically from one or more data sources (e.g., hospitals in geographically disparate cities), which can lead to datasets that are biased and not representative, i.e., the collected dataset does not accurately reflect the distribution of demographics present in the true population. 
This is a concern because subgroups within the population can be under- or over-represented in a dataset, which may harm generalizability and lead to an unequal distribution of benefits and harms from downstream tasks that use such datasets (e.g., algorithmic bias in medical decision-making algorithms). 
In this paper, we assess the relationship between dataset representativeness and group-fairness of classifiers trained on that dataset.
We demonstrate that there is a natural tension between dataset representativeness and classifier fairness; empirically we observe that training datasets with better representativeness can frequently result in classifiers with higher rates of unfairness.
We provide some intuition as to why this occurs via a set of theoretical results in the case of univariate classifiers.
We also find that over-sampling underrepresented groups can result in classifiers which exhibit greater bias to those groups.
Lastly, we observe that fairness-aware sampling strategies (i.e., those which are specifically designed to select data with high downstream fairness) will often over-sample members of majority groups.
These results demonstrate that the relationship between dataset representativeness and downstream classifier fairness is complex; balancing these two quantities requires special care from both model- and dataset-designers.
\end{abstract}

\section{Introduction}
Representation biases, where certain subpopulations appear more, or less, frequently in a dataset than they do in a target population of interest is a foundational problem.  
Failure to adequately diversify data can induce numerous downstream effects, such as the creation of data-based models that are unfair in their performance \cite{kearns2018preventing, feldman2015certifying, abernethy2020active}. 
Yet, this is not a  recent phenomenon.  
The Framingham Heart Study (FHS), initiated in 1948, provided revolutionary insight into cardiovascular disease over time.
It enabled the development of disease risk prediction tools like the Framingham Risk Score that were widely applied in practice to recognize and proactively manage patients at risk for coronary heart disease \cite{wilson_prediction_1998}.
However, the original study cohorts were nearly all of white race \cite{leaverton_representativeness_1987}, until more racially diverse participants started to be recruited in 1994 \cite{mahmood_framingham_2014}.
Analyses found that applying FHS risk coefficients yielded inaccurate risk predictions for non-white populations \cite{liao_prediction_1999,dagostino_validation_2001}.
Using the same risk factors, but deriving the actual risk coefficients from racially diverse cohorts, yielded comparable predictive performance across racial groups \cite{hurley_prediction_2010}.
These findings indicate that disparities in group-wise predictive accuracy stemmed from insufficient representation of minority groups in FHS.
In high-stakes domains like healthcare, these inaccuracies can cause quantifiable harm to underrepresented groups \cite{national_academies_of_sciences_why_2022}.
Though known for some time, this phenomenon has become increasingly accentuated because of an increased societal reliance on automated systems learned via aggregated datasets.
Studies of genomic datasets have shown vast differences in downstream predictive performance between highly- and underrepresented groups
\cite{schoeler2023participation,bentley2020evaluating,sirugo_missing_2019}.
Nevertheless, the relationship between subgroup-specific representation and downstream performance has not been fully explored.

Dataset representativeness yields multiple different types of benefits.
As noted above, representative datasets promote generalizability and validity of findings to the entire population of interest.
Researchers often aim to discover generalizable results, while large biomedical datasets, like the \textit{All of Us} Research Program, have increasingly focused on recruiting diverse populations \cite{mapes_diversity_2020}.
In addition to its downstream benefits, representativeness engenders legitimacy, as seen in policymaking \cite{arnesen_legitimacy_2018}.
A putative mechanism for this effect is that representativeness supports procedural fairness, the concept of equal treatment of individuals by systems and processes \cite{burke_court_2007}.
Conversely, unrepresentative biomedical datasets may undermine trust in the research enterprise \cite{national_academies_of_sciences_why_2022}.
We measure representation intuitively through first-order information of the \emph{true population} and the constructed dataset, specifically via the difference between the average occurrence of each sensitive feature in the population and in the dataset. 
We formulate this concept rigorously in Definition \ref{def:rep}, where a perfectly representative dataset (i.e., one where the proportions of every group are identical in the dataset and population) would have zero difference.

We focus on the practical example of \emph{multi-site} data collection, where data or individuals are sampled from a set of $m$ sites across a limited number of iterations $T$.
Multi-site projects like PCORnet and the \textit{All of Us} Research Program enable unprecedented access to human subjects data and represent billions of dollars in investment \cite{forrest_pcornet_2021, the_all_of_us_research_program_investigators_all_2019}.
The response distribution, affected by both underlying site demographics (which may be known or estimated \textit{a priori}) and by the willingness of demographic groups to participate in the study, at each site starts as an unknown.
With each iteration, the data-collector selects a site (or sites) to obtain data from, and then yields a number of examples to add to their dataset.

In this study, we address several contemporary issues surrounding multi-site dataset construction.
First, we propose an algorithm to construct a representative dataset from several available sites and compare it to baselines.
Then, we assess how varying group representation affects algorithmic fairness and how the multi-site framework alters the representativeness-fairness relationship.
Finally, we analyze cases where more representative datasets do not yield fairer classifiers and discuss alternative approaches to improve fairness and representativeness.

Our paper is organized as follows:
\begin{itemize}
    \item In Section \ref{sec:prelim} we formalize the problem collecting a \emph{representative} dataset via site-based sampling.
    \item Section \ref{sec:how_to_sample} discuss the question of \emph{how} to sample from sites; we propose an algorithm for representative sampling (Algorithm \ref{alg:PBRS}), and an adapted version of \cite{abernethy2020active} for fair site-based sampling.
    \item Next, in Section \ref{sec:single_var} we begin our investigation into the relationship between fairness and representatives with a case study on single variable classifiers.
    \item Section \ref{sec:method} outlines our experimental methodology. 
    \item Lastly, Section \ref{sec:exp} provides our primary experimental results: showing the effectiveness of our proposed algorithm for representative sampling , as well as investigate the relationship between representativeness and fairness.
\end{itemize}
\subsection{Related Work}
There have been numerous investigations into what it means for a given collection of samples to be representative or for an algorithm to be fair.
Representativeness is typically defined either as 1) a statistical distance from a goal or true distribution \cite{he_multivariate_2016,celis_data_2020,qi2021quantifying} or 2) a measure of coverage of attribute combinations \cite{asudeh2019assessing,jin_mithracoverage_2020}.
\cite{shahbazi_representation_2023} provide an extensive survey on methods to measure and address representation bias.
When the target population is unknown, but researchers are still interested in assessing group disparities, sampling from groups equally is an efficient method \cite{singh_measures_2023}.

When individuals may be selected according to their attributes, methods for selecting representative cohorts have been proposed for specific use cases: hiring processes, citizens' assemblies, and record selection from a single database \cite{huppenkothen2020entrofy,flanigan2021fair,borza2022representativeness}.
Given uncertain site-specific population distributions, our problem of representative dataset construction via sequentially sampling sites is similar to the multi-armed bandit problem \cite{auer2002finite,BC-12} with concave reward structure \cite{agrawal2014bandits}.
The most closely related work to ours in this regard is by \cite{nargesian2021tailoring}, who utilize a bandit-based approach to achieve a desired attribute distribution in multi-site data collection when faced with uncertain site attribute distributions.
This algorithm constructs a reward function with higher values for samples containing individuals from minority groups, in order to achieve a desired distribution.

Like with representation, numerous definitions have been proposed for algorithmic fairness.
Many definitions of fairness originate from Rawlsian theories of justice, which eschew inequalities between individuals \cite{rawls_justice_1958}.
\cite{dwork_fairness_2012} adapted this concept to ensure similar individuals receive similar algorithmic outcomes.
Similarly, \cite{hardt_equality_2016} defined fairness through equal odds and equal opportunity, requiring the equalization of true positive rates and false positive rates between demographic groups, respectively.
Parity based measures of fairness now exist for every common decision and prediction measure for an algorithm \cite{mitchell_algorithmic_2021}. 
However, there is little consensus on how to best measure algorithmic fairness, and different measures can be impossible to simultaneously satisfy except under trivial conditions \cite{kleinberg_inherent_2017}.
Some definitions of fairness (e.g., worst-group performance) have been used to guide data collection \cite{abernethy2020active,shekhar2021adaptive,niss2022achieving}.
During the data collection process, these approaches presume both the hypothesis class of the downstream model as well as the predictive task.
Similarly, post-hoc subgroup re-balancing may improve algorithmic fairness in certain downstream tasks \cite{idrissi2022simple,zhang2022improving}, but post-hoc corrections may also severely impact predictive accuracy \cite{woodworth_learning_2017}.
In practice, datasets are used for a multitude of model types and predictive tasks, and as such, a dataset which is \emph{fair} for one combination of predictive task and model may be unfair for other types.

The relationship between representation and fairness is less explored than its two constituent concepts.
On the one hand, it is well-known that classifier performance tends to be poor for underrepresented groups and that increasing representation of these groups in training data can improve performance \cite{bentley2020evaluating,li_data-centric_2022,wang_overwriting_2023}.
Yet, these notions do not establish an optimal level of representation to best support fairness.
A na\"ive approach may be to equalize group proportions or sample more data points, but these techniques do not necessarily improve fairness \citep{li_data-centric_2022}.
\cite{chen_why_2018} propose a decomposition of discrimination --- a generalization of unfairness --- into bias, variance, and noise terms, each with unique remediation strategies: increasing model capacity, sampling from the disadvantaged group(s), and collecting additional features.
While this is a useful and intuitive way to categorize causes of unfairness, determining which factor(s) drives discrimination relies on having a Bayes-optimal classifier, which is often computationally impractical.

\section{Preliminaries}\label{sec:prelim}
To formalize our setting, let $\pazocal{X} \times\pazocal{A} \times \pazocal{Y}$ be a domain of features $\pazocal{X}\subset \mathbb{R}^{\ell}$, sensitive features $\pazocal{A}\subset \mathbb{R}^d$  and binary labels ${\pazocal{Y}\equiv\{0, 1\}}$.
Let $D$ be a distribution over $\pazocal{X}\times\pazocal{A}\times\pazocal{Y}$, i.e., $D$ is the \emph{true} population distribution.
The data collector does not know the distribution $D$, but may know its mean.
Let $\pazocal{S} = \{S_1, \ldots, S_m\}$ be the set of $m$ \emph{sites}, where each site $S_j$ is associated with an underlying site-specific population distribution $D_j$ over $\pazocal{X}\times\pazocal{A}\times\pazocal{Y}$.
Importantly, the distribution $D_j$ for every site $j\in[m]$ 
is unknown to the data collector.
Over the course of $T$ timesteps, the data collector will sequentially 
recruit samples from sites $\pazocal{S}$, with the objective of building a \emph{representative} final dataset.
Each sample from site $S_j^{(t)}$ constitutes a draw $(\boldsymbol{X}, \boldsymbol{A}, \boldsymbol{Y})^{(t)} \sim D_j$. 
After $T$ rounds, the data collector has a dataset 
\(
    (\boldsymbol{X}, \boldsymbol{A}, \boldsymbol{Y}) = \bigcup_{t=1}^T(\boldsymbol{X}, \boldsymbol{A}, \boldsymbol{Y})^{(t)}.
\)
Given a \emph{target} demographic vector $\boldsymbol{v}\in\mathbb{R}^d$, which represents the ideal mean of $\boldsymbol{A}^{(T)}$, the data collector aims to sample such that $\textrm{avg}(\boldsymbol{A}^{(T)})$ is as close to $\boldsymbol{v}$ as possible.
Thus, we conceptualize representativeness as inversely proportional to the distance from $\textrm{avg}(\boldsymbol{A}^{(T)})$ to $\boldsymbol{v}$; as this distance decreases, representativeness increases.
For example, suppose there are two binary features of interest: gender (Male or Female) and age (Young or Old). A target vector of $\boldsymbol{v} = \langle 0.3, 0.7\rangle$ implies that an ideal dataset is 30\% Male and 70\% Young. 
Therefore if $\pazocal{M}$ is the $\ell_1$-norm then a dataset which is 25\% Male and 60\% Young would be $0.15$-distant with respect to $\boldsymbol{v}$. 
We next formally define representativeness.

\begin{definition}
    \label{def:rep}
    \textbf{(Representativeness):} The representativeness of a dataset $(\boldsymbol{X}, \boldsymbol{A}, \boldsymbol{Y})$ with respect to a target demographic vector $\boldsymbol{v}\in \mathbb{R}^d$ and distance metric $\pazocal{M}$ is inversely proportional to
    $\pazocal{M}\big(\boldsymbol{v}, ~\frac{1}{|\boldsymbol{A}|}\sum_{\boldsymbol{a}\in \boldsymbol{A}}\boldsymbol{a} \big)$, where $\frac{1}{|\boldsymbol{A}|}\sum_{\boldsymbol{a}\in \boldsymbol{A}}\boldsymbol{a}$ is the mean vector of the demographics in the dataset.
\end{definition}
Given target vector $\boldsymbol{v}$, the objective of sampling the most representative dataset can be expressed as 
\begin{equation}
    \label{eq:obj}
    \min_{(\boldsymbol{X}, \boldsymbol{A}, \boldsymbol{Y})}\pazocal{M}\bigg(\boldsymbol{v},  \frac{1}{|\boldsymbol{A}|}\sum_{\boldsymbol{a}\in \boldsymbol{A}}\boldsymbol{a} \bigg) \\
    \textrm{s.t.}(\boldsymbol{X}, \boldsymbol{A}, \boldsymbol{Y}) = \bigcup_{t=1}^T(\boldsymbol{X}, \boldsymbol{A}, \boldsymbol{Y})^{(t)}.
\end{equation}
We limit $\pazocal{M}$ to distance measures which are convex in the collected set of sensitive features $\boldsymbol{A}$, including all $\ell_p$-norms with $p\geq 1$ and KL-divergence.
It should be recognized that a key challenge with representative sampling is that the objective in Problem \ref{eq:obj} is not supermodular, even for convex $\pazocal{M}$, as a function of $\frac{1}{|\boldsymbol{A}|}\sum_{\boldsymbol{a}\in \boldsymbol{A}}\boldsymbol{a}$. 
This is due to the nonlinear nature of the average $\frac{1}{|\boldsymbol{A}|}\sum_{\boldsymbol{a}\in \boldsymbol{A}}\boldsymbol{a}$, with respect to samples.

\section{Convex Formulation and Prior-Based Sampling}\label{sec:how_to_sample}
In this section, we first demonstrate how the data collector's sampling problem can be formulated through the framework of multi-armed bandit with concave reward (convex loss in our case). 
Utilizing this particular problem structure, we 
present our algorithm for constructing representative datasets. 
Our strategy for optimizing this objective is to provide a modified form of the objective in Equation \ref{eq:obj} which is convex with respect to the samples collected at each time step.
To do this, we first note that each iteration returns $k$ data points\footnote{The convex formulation holds when, in expectation, each iteration yields $k$ data points.}, and thus the final dataset will consist of $Tk$ examples, and the average demographic vector of the dataset can be written as 
\begin{equation}
    \frac{1}{|\boldsymbol{A}|}\sum_{\boldsymbol{a}\in \boldsymbol{A}}\boldsymbol{a} 
    = \frac{1}{T}\sum_{t=1}^T \sum_{\boldsymbol{a}\in \boldsymbol{A}^{(t)}}\frac{\boldsymbol{a}}{k} 
    = \frac{1}{T}\sum_{t=1}^T \textrm{avg} \big(\boldsymbol{A}^{(t)}\big)
\end{equation}
where $\textrm{avg}\big(\boldsymbol{A}^{(t)}\big)$ is the average demographic vector present in the sample $\boldsymbol{A}^{(t)}$ collected at time $t$.

With this fact, the data collector's objective can be expressed as a function simply of the sum of the means from each sample,

\begin{equation}
    \label{eq:obj2}
     \min_{\boldsymbol{A}} \pazocal{M}\bigg(\boldsymbol{v},  ~\frac{1}{T}\sum_{t=1}^T\textrm{avg}\big(\boldsymbol{A}^{(t)}\big)\bigg)
\end{equation}

\begin{theorem}
\label{thm:convex}
The objective in Equation \ref{eq:obj2} is convex with respect to the sample values $\textrm{avg}\big(\boldsymbol{A}^{(t)}\big)$ and has an equivalent optimal value with Equation \ref{eq:obj} after all $T$ rounds are completed. 
\end{theorem}

We defer this proof to appendix \ref{sup:proofs}.
Since the samples returned by each site at time $t$ can now be thought of as a single vector $\textrm{avg}\big(\boldsymbol{A}^{(t)}\big)$, and the loss function $\pazocal{M}$ is convex with respect to those sample vectors, the problem of representative sampling can be naturally formulated as a multi-armed bandit problem with convex loss.
We next discuss Bayesian sampling procedure which can capitalize on both this convex formulation as well as site-wise prior information.

\subsection{Prior-based Bayesian Representative Sampling (PBRS)}
\begin{algorithm}[tbh!]
\small
    \caption{Prior-based Bayesian Representative Sampling (PBRS).}
    \label{alg:PBRS}
    \KwData{Sites $\pazocal{S}$, classifier $\pazocal{F}$, representativeness metric $\pazocal{M}$}
    Initialize priors for group demographics at each site $s_j$ (inverse Wishart distribution): $\pazocal{W}_j^{-1}$;\\
    \For{$t = 1 \ldots T$}{
        \For{\textrm{site} $s_j$ in $S$}{
        sample mean and covariance of group demographics,  $\boldsymbol{\theta}_j, \boldsymbol{\Sigma}_j \sim \pazocal{W}_j^{-1}$:\\
        simulate sampling from site $s_j$ via $\mathbf{a}_j\sim \mathcal{N}(\boldsymbol{\theta}_j, \boldsymbol{\Sigma}_j)$;\\
        improvement$_j\gets$ $\pazocal{M}(\mathbf{A}) - \pazocal{M}(\mathbf{A} \cup \mathbf{a}_j)$;\\
        }
    $j^* \gets$ site with the largest improvement$_j$;\\
    sample $(\mathbf{X}^{(t)}, \mathbf{A}^{(t)}, \mathbf{Y}^{(t)})$ from site $s_j$;\\
    update dataset $(\mathbf{X}, \mathbf{A}, \mathbf{Y}) \cup= (\mathbf{X}^{(t)}, \mathbf{A}^{(t)}, \mathbf{Y}^{(t)})$;\\
    }
    $\textbf{return}\; (\boldsymbol{X}, \boldsymbol{A}, \boldsymbol{Y})^{(t)}$;
\end{algorithm}

\begin{algorithm}[tbh!]
\small
    \caption{Fair Arm-Based Sampling}
    \label{alg:fair}
    \KwData{Sites $\pazocal{S}$, classifier $\pazocal{F}$, loss function $\mathcal{L}(\pazocal{F}(\boldsymbol{X}), \boldsymbol{Y})$}
    \KwResult{dataset $(\mathbf{X}, \mathbf{A}, \mathbf{Y})$}
    randomly sample initial data $(\mathbf{X}, \mathbf{A}, \mathbf{Y})$;\\
    \For{$t = 1 \ldots T$}{
        train $\pazocal{F}$ using current data $(\mathbf{X}, \mathbf{A}, \mathbf{Y})$;\\
        $g^* \gets $ group with with the highest loss w.r.t, $\pazocal{F}$, and $\mathcal{L} $;\\
        $s_j^* \gets$ site with the largest expected proportion of $g^*$;\\
        sample new data $(\boldsymbol{X}^{(t)}, \boldsymbol{A}^{(t)}, \boldsymbol{Y}^{(t)})$ from site $s_j^*$
        update dataset $(\mathbf{X}, \mathbf{A}, \mathbf{Y}) \cup= (\mathbf{X}^{(t)}, \mathbf{A}^{(t)}, \mathbf{Y}^{(t)})$;\\
    }
    $\textbf{return}\; (\boldsymbol{X}, \boldsymbol{A}, \boldsymbol{Y})^{(t)}$;
\end{algorithm}

Before outlining the details of our algorithm, we first discuss the motivation behind PBRS (Alg. \ref{alg:PBRS}), which is twofold.
First, in many real-world domains where representativeness is a salient issue, a wealth of summary data is available, which allows data collectors to form reasonably accurate priors over the distributions at each site. 
Second, the Bayesian nature of our approach always for dynamic control over how aggressively the prior distributions are updated after each sample, this is particularly useful in settings where the distributions at sites may change over time (a common occurrence in the real-world), such shifts are discussed in Section \ref{site_var}.
The full PBRS algorithm (Alg. \ref{alg:sup_PBRS}) is in appendix \ref{sup:algs}.

PBRS works by maintaining an estimate of the distribution of groups at each site $D'_j$, which corresponds to a multinomial distribution, when sensitive features are binary and a multivariate-normal distribution when sensitive features are continuous.
In the former $D'_j = M_d(k, p_{j,1}, \ldots, p_{j,d})$, where $p_{\ell}$ gives the probability that an individual sampled from site $j$ will have sensitive feature $\ell$ equal to $1$.
In the latter, $D'_j = \boldsymbol{\pazocal{N}}(\boldsymbol{\theta}_j', \boldsymbol{\Sigma}_j')$ where $\boldsymbol{\theta}_j'$ and $\boldsymbol{\Sigma}_j'$ are the mean and covariance of sensitive features at site $j$.
In both cases, each distribution is initialized via a prior estimate of the true distribution at site $j$.
In the case that no prior is provided, a \emph{default} prior can be induced by either assigning uniform values to each parameter (e.g., $\boldsymbol{\theta}_j' = \mathbf{0}$ and $\boldsymbol{\Sigma}_j' = I_d$), or as values from the target vector $\boldsymbol{v}$ (e.g., $p_{j, \ell} = \boldsymbol{v}[\ell]$ for all $\ell \in [d]$).
Throughout the course of constructing the dataset, the samples obtained at each time step can be used to update these distributions to more accurately reflect the true distribution of each site. 
To do this, we use the conjugate prior of each distribution to iterative update the estimation $D_j'$. 
In the case of binary group features, the conjugate prior is represented by a Dirichlet distribution $\textrm{Dir}(d, \alpha_{j, 1}, \ldots, \alpha_{j, d})$, and in the case of continuous group features, the conjugate prior is represented by an inverse Wishart distribution $\pazocal{W}_j^{-1}(\boldsymbol{\theta}_j', \boldsymbol{\Psi}_j, n_j)$.

At each time step $t$, the estimated distribution $D_j'$ is induced by sampling parameters from the corresponding conjugate prior, and is then used to compute the expected improvement  to $\pazocal{M}\big(\boldsymbol{v}, \textrm{avg}(\boldsymbol{A})\big)$ for each site.
PBRS selects the site $j^*$, corresponding to the maximum expected improvement. 
The sample from site $j^*$ is then used to update conjugate prior.
To better anticipate the possibility for site bias, we incorporate a hyperparameter $\beta \geq 1$ which modifies the procedure through which conjugate distributions are updated by increasing the strength of samples from minority groups by a factor of roughly $\beta^{(1-t/T)}$.
This hyperparameter incentivizes PBRS to more aggressively search for sites which yield individuals from minority groups, thus helping to circumvent site bias towards those groups.

\subsection{Distributed Prior-based Bayesian Representative Sampling (D-PBRS)}
D-PBRS (Alg. \ref{alg:sup_PBRS}, appendix \ref{sup:algs}) modifies PBRS to allow multiple sites to be sampled from simultaneously in a single timestep, still limited to $k$ total samples per timestep.
D-PBRS distributes the budget $k$ according to a vector $\rho$, which is selected to maximally decrease $\pazocal{M}$ given all previously collected samples, with the constraint that $\Sigma\rho=1$. 
In the sampling step, $k$ total samples  are divided among the sites according to $\rho$ with fractional sample allocations rounded down, and assigned to the site that minimizes $\pazocal{M}$. For example, int he case of two sites and a budget of  $k=40$, $\rho=\langle.75, .25\rangle$ implies collecting $30$ samples from the first site, and $10$ from the second site.

\subsection{Fair Arm-Based Sampling}
We introduce a third arm sampling procedure (Alg. \ref{alg:fair}), one designed to optimally improve minmax algorithmic fairness.
We enact this goal by first training a classifier on the available dataset, then evaluating its group-specific performance on a set of validation data.
Next, we identify the group with the lowest AUC and sample from the arm with the highest proportion of that group.
This algorithm represents an adaptation of previous work by \cite{abernethy2020active} and \cite{shekhar2021adaptive} to our arm-based selection process.
The full fair sampling algorithm (Alg. \ref{alg:sup_fair}) is in appendix \ref{sup:algs}.

\section{Univariate Case Study}\label{sec:single_var}
To build intuition for the relationship between representatives and fairness we examine classification when the predictive features are single variable, i.e., $x\in \mathbb{R}$.
Note that univariate classification and multivariate classification are equivalent in the sense that $x$ can correspond to the output of a score function applied to the multidimensional feature $\mathbf{x}$, i.e. $x = h(\mathbf{x})$.

We being by demonstrating the existence of a trade-off between fairness and representatives. This trade-off stems from relative \emph{difficulty} in learning the joint, $\mathbb{P}\big(y=1| x \big)$; that is, as the relative difficulty of learning the joint increases, so does the trade off between representatives and fairness.

To capture the difficulty of learning the joint, let the relationship between $x$ and $y$ be defined as $y = \mathbb{I}\big[ x + \varepsilon_g \geq \theta_g \big]$ where $\varepsilon_g \sim \mathcal{N}(\mu_g, \sigma_g)$ gives the noise of the label $y$. Let $D_g$ be the distribution over features and labels for group $g$. 

\begin{theorem}\label{thm:unfair_1}
    Suppose there are $n_0$ and $n_1$ samples collected from groups $g=0$ and $g=1$ respectively.
    Let $\pazocal{F}$ be the optimal classifiers learned on these samples (in terms of expected accuracy). 
    Let $\delta = \textrm{error}(\pazocal{F}, D_0) - \textrm{error}(\pazocal{F}, D_1) $, i.e., the difference in accuracy between groups. 
    Then $\mathbb{E}\big[ \delta \big] = \sqrt{2/\pi}\big(\sigma_0\sqrt{1/n_0} - \sigma_1\sqrt{1/n_1}\big)$
\end{theorem}

The key takeaway from Theorem \ref{thm:unfair_1} is that it allows us to quantify expected unfairness $\mathbb{E}[\delta]$ in terms of both the number of samples collected from each group $n_0, n_1$ and the relatively noisiness of each groups' labels $\sigma_0, \sigma_1$. The expression of expected unfairness immediately yields the following result. 
\begin{proof}
    We defer the proof to appendix \ref{sup:proofs}.
\end{proof}
\begin{theorem}\label{thm:unfair_2}
    Suppose the optimal classifier trained on $n_0$ samples from group $0$ and $n_1$ samples of group $1$ has an unfairness of at most $\delta$, then it must be the case that 
    
    $n_1 \bigg(\frac{\sigma_0}{\delta\sqrt{\pi/2 n_1} + \sigma_1}\bigg)^2 \leq n_0$ 
   and 
   $n_0 \bigg(\frac{\sigma_1}{\delta\sqrt{\pi/2 n_0} + \sigma_0}\bigg)^2\leq n_1$
\end{theorem}
This theorem indicates that in order to limit the accuracy-disparity between groups to be no greater than $\delta$, the sampling rates between the two groups cannot be too different (where ``too different" is dictated by the relative noise levels of each group, $\sigma_0, \sigma_1$).
\begin{proof}
    We defer the proof to appendix \ref{sup:proofs}.
\end{proof}
\begin{theorem}\label{thm:unfair_3}
    In order to achieve an unfairness of $0$, the sample ratio between the two groups must be $n_0 = (\sigma_0^2 / \sigma_1^2) n_1$. 
\end{theorem}
This theorem demonstrates that achieving an expected unfairness close to $0$ may not be possible within a budge of $m$ total samples (i.e., $m=n_0+n_1$). To see this, imagine a case in which $\sigma_0^2 / \sigma_1^2 > m+1$, i.e., group $0$ has vastly higher noise than group $1$. Then, the sampler will not be able to collect enough samples to ensure that $(\sigma_0^2 / \sigma_1^2) n_1 < n_0$.
\begin{proof}
    This result follows directly from Theorem \ref{thm:unfair_1}.
\end{proof}
\section{Methodology}\label{sec:method}
\begin{table*}[tb]
\centering
\begin{small}
    \begin{tabular}{llllr}
        \toprule
         Dataset        &  Sensitive Features & Target Feature & Location & Size \\
         \midrule
         Law School     &  Race, Gender, Age, Family Income & Pass Bar & School & 20,454 \\
         Lending Club   &  Housing Status, Occupation & Repay Loan & ZIP Code & 124,040 \\
         Intensive Care & Race, Gender, Age & ICU Recovery & Hospital & 48,612 \\
         Texas Salary & Race, Gender & Earn $\geq$ \$75k & Office & 142,981 \\
         Adult Income & Race, Gender, Age & Income $\geq$ \$50k & --- & 46,447 \\
         Community Crime & Race Proportion & Low Crime Risk & --- & 1,994 \\
         \bottomrule
    \end{tabular}
    \caption{Dataset Details. All Sensitive Features are Treated as Binary Indicators.}
    \label{tab:data_details}
    \end{small}
\end{table*}

\subsection{Datasets}

We evaluated our methodology on six commonly-used datasets: 
1) \textbf{Law School}  \cite{wightman1998lsac}, 
2) \textbf{Lending Club} \cite{lending2020club}, 
3) \textbf{Intensive Care} \cite{pollard2018eicu}, 
4) \textbf{Texas Salary} \cite{texas21},
5) \textbf{Adult Income} \cite{misc_adult_2},
and 6) \textbf{Community Crime} \cite{misc_communities_and_crime_183}.
Each dataset contains features that differentiate between groups of interest, as well as location-based information (Tab. \ref{tab:data_details}) when available.
For datasets 1-4, we partition the dataset into $m$ disjoint sets sharing the same location, inducing $m$ sites (i.e., arms).
The Law School, Intensive Care, and Texas Salary datasets include location information corresponding to \emph{actual} sites, such as the student's law school.
For the Lending Club dataset, we induce sites by U.S. state.
The Adult Income and Community Crime datasets do not have applicable location information, so they are not used to evaluate our sampling algorithms.
Nevertheless, these two datasets have well-documented algorithmic fairness limitations, making them ideal case studies for our fairness analyses.
Sites with fewer than 1,000 records were excluded from analysis due to small sample size limitations.

\subsection{Sampling Procedure and Algorithms}
For a target demographic vector $\boldsymbol{v}$ and a distance measure $\pazocal{M}$, we iteratively select a site (or mix of sites) and receive $k$ data points $(\mathbf{x}, \mathbf{a}, \mathbf{y})$ randomly sampled from the partition corresponding to that site.
After repeating the process $T$ times, we combine the $T \cdot k$ data points into a single dataset and compute the distance between the target demographic vector and the average demographics of the constructed dataset $\pazocal{M}\big(\boldsymbol{v}, \textrm{avg}(\boldsymbol{A})\big)$.
To demonstrate the improved efficacy of \textbf{PBRS} (BY(H) and BY(L) for high- and low-noise priors) and \textbf{D-PBRS} (DS(H) and DS(L) for high- and low-noise priors), we compare to three baselines: 
1) \textbf{$\varepsilon$\emph{-Greedy}} ($\varepsilon$GRD): which randomly selects a site with probability $\varepsilon$ and otherwise selects the site which has the maximum expected decrease in error;
2) \textbf{UCB-LCB} \cite{agrawal2014bandits} (UCB): which is a UCB-based algorithm \cite{auer2002finite} for solving multi-armed bandit problems with convex loss;
and 3) \textbf{OL-Vec} \cite{kesselheim2020online} (VEC), which derives a one dimensional function to approximate the distance measure $\pazocal{M}$ and uses online convex minimization to select the site at each timestep.
In addition to the aforementioned baselines, we also compared to random site selection \textbf{Random} (RND), and \textbf{OPT}, a policy that has full information and selects the site corresponding to the maximum expected decrease in error. 
This baseline serves as the best possible \emph{myopic} sampling scheme when the data collector is limited to a single site per timestep.
To test our representative sampling algorithms, we analyzed a setting in which there are 20 arms (achieved via either randomly subsampling or duplicating sites, depending on the number of sites in the dataset), 50 time steps, and sample sizes of $k=40$ individuals.
Based on this setup, the constructed dataset corresponds to 2,000 examples.
We use a class-balanced target vector, $\boldsymbol{v} = \langle .5, \cdots, .5 \rangle$, and average performance across 100 experiments.
To measure how effective each sampling algorithm is at producing a representative dataset with respect to a target demographic vector $\boldsymbol{v}$, we use the $\ell_2$-norm,  $\pazocal{M}\big(\boldsymbol{v}, \textrm{avg}(\boldsymbol{A})\big) = \|\boldsymbol{v} - \textrm{avg}(\boldsymbol{A})\|$). 

\subsection{Site Variations} \label{site_var}
\textbf{No bias} is our baseline. In this setting, site response distributions are induced by the location-based partitions and do not change over time.

\textbf{Response bias} occurs when certain demographics appear at sites with disproportionately high (or low) frequencies compared to other groups. 
For example, as shown in \cite{aba22} the ratio of individuals identifying as ethnic minorities is substantially lower at the majority of law schools compared to the population.
Response bias can be modeled using coefficients $\lambda \in \mathbb{R}_{\geq 0}$ and 
$\gamma \leq m$,  where members of majority groups are $\lambda$-times more likely to respond at $\gamma$ sites compared to their base response rate at those sites. 
For example suppose there is one binary feature (i.e., two groups), $\gamma=m/2$, and $\lambda=4$, then individuals from the majority group are $4$-times more likely to appear in a sample from half of the sites. 
The no variation setting is recovered when $\lambda=1$.
We evaluate the representativeness of the final datasets constructed by the tested algorithms across a range of $\lambda$ from $0.1$ to $10$.

Lastly, \textbf{causal distribution shifts} occur when demographic distributions at each site change over time as the result of the data collector's decisions. 
When selection is desirable (e.g., monetary compensation for participating in trials), individuals may modify their behavior in order to be selected. 
Causal distribution shifts affect response probability $p$ of each individual at site $j$ with coefficient $\alpha \in \mathbb{R}_{\geq0}$ s.t. $p_{post} = p_{pre}^{1 + \alpha \times \rho_j}$.
We evaluate the representativeness of the final datasets constructed by the tested algorithms across a range of $\alpha$ from $0.1$ to $10$ using $\lambda = 2$ such that there is a response bias to causally magnify.

\subsection{Arm Sampling and Downstream Fairness}
To asses data quality with respect to downstream tasks, we compare the predictive efficacy of datasets produced by optimal arm-based sampling with OPT, arm-based fair sampling, location-agnostic stratified random sampling (SRS), and fair direct sampling.
Each data domain is partitioned into four folds, generating four 75\%/25\% train/test splits.
Then, 200 desired sensitive feature group fractions linearly spaced from 0 to 1 are generated.
For SRS, a 2000-record sample is selected from the training set for each sensitive feature fraction.
In arm-based sampling, the training set is partitioned by site, then 2000-record samples are generated for each sensitive feature fraction using OPT.
Unlike the representative sampling algorithms, the fair sampling algorithms do not target a specific sensitive feature group balance.
Instead, a group balance emerges secondarily as a result of selecting records from the sensitive feature group (either $G_0$ or $G_1$, as each analysis studies only one sensitive feature at a time) with lower performance.
Fair arm-based sampling is achieved via algorithm \ref{alg:fair} and is repeated five times in each train/test fold.
Fair direct sampling adapts the algorithm proposed by \cite{abernethy2020active} to successively identify the worse-performing group and draw mini-batches of 5 examples from it.
We initialize the fair direct sampling algorithm with four examples, one for each of the two sensitive features and two labels.

Two model classes, Logistic Regression (LR) and  Gradient Boosted Decision Tree Classifiers (GBC), are fit to a single binary prediction task $f:\boldsymbol{X}\rightarrow\boldsymbol{Y}$ where $\boldsymbol{X}$ are all non-sensitive features and $\boldsymbol{Y}$ is the target feature in table \ref{tab:data_details}.
Models are trained on each sampled dataset using default hyperparameters in scikit-learn and weighted to be class balanced in $\boldsymbol{Y}$ because of inherent label imbalance in our training datasets.
Model AUCs are computed for the population and groups $\boldsymbol{A}=0 \;(G_0)$ and $\boldsymbol{A}=1 \;(G_1)$. 
We evaluate for algorithmic fairness by assessing the disparity in AUC between groups $G_0$ and $G_1$.

\subsection{Fairness and Complexity Analysis}
To delve further into the relationship between representation and fairness, we study three datasets with known unfairness: Law School, Adult Income, and Community Crime.
We start similarly to our previous studies of arm sampling and downstream fairness, with some key modifications.
Because Adult Income and Community Crime do not have locations, we do not partition these datasets into sites and, thus, we only apply SRS to sample for representation.
Thus, we alleviate the restriction that each site must have 1,000 records.
Because this analysis is more flexible with respect to record selection, we partition the datasets into ten folds (90\% train / 10\% test) and average results across these folds.
To accommodate large differences in record counts between these datasets, we fix the training set size to the size of the smallest sensitive feature group in the training fold.
This is the largest possible training set that still allows a cohort to be made up entirely of one group.
We build training sets for 21 linearly spaced group $G_1$ proportions from 0 to 1, representing 5\% proportion increments within the training data.
As before, we train GBC models to the binary prediction task $f:\boldsymbol{X}\rightarrow\boldsymbol{Y}$.
In addition to presenting population and group-specific AUCs for various group proportions, we present true positive rates (TPR, i.e., sensitivity) and true negative rates (TNR, i.e., specificity).
TPR and TNR parity are widely-used measures of algorithmic fairness and supplement AUC parity for the purposes of this analysis \cite{hardt_equality_2016, mitchell_algorithmic_2021}.

When modifying group representation does not decrease a significant performance disparity between groups, other factors must be limiting fairness.
We theorize difficulty of learning plays a significant role in driving unfairness (Thm. \ref{thm:unfair_1}).
To assess this theorem empirically, we evaluate the fairness of classifiers differing in their ability to capture complex relationships between features and labels.
The capability of individual decision trees to capture complex relationships is driven primarily by the number of internal nodes, which is in turn driven by the depth of the tree \cite{leboeuf_decision_2020, bentejac_comparative_2021}.
To control GBC complexity, we limit both maximum tree depth from 1 to 8 (default: 3) and the number of estimation steps from 1 to 500 (default: 100); then we assess models constrained simultaneously by both of these limits.
We hypothesize a more complex classifier can capture more difficult learning relationships, and subsequently improve AUC, TPR, and TNR parity.
We measure the group difference (i.e., $G_0 - G_1$) of AUC, TPR, and TNR for each sensitive feature in each of the three datasets outlined in this section.
It is well-known that certain fairness constraints on models can harm model accuracy \cite{chen_why_2018}, so we also assess the total test set AUC of our classifiers as we vary model complexity.

\section{Experimental Results}\label{sec:exp}
\subsection{Sampling Algorithm Evaluation}
\begin{figure}[bht!]
    \centering
    \includegraphics[width=0.45\linewidth]{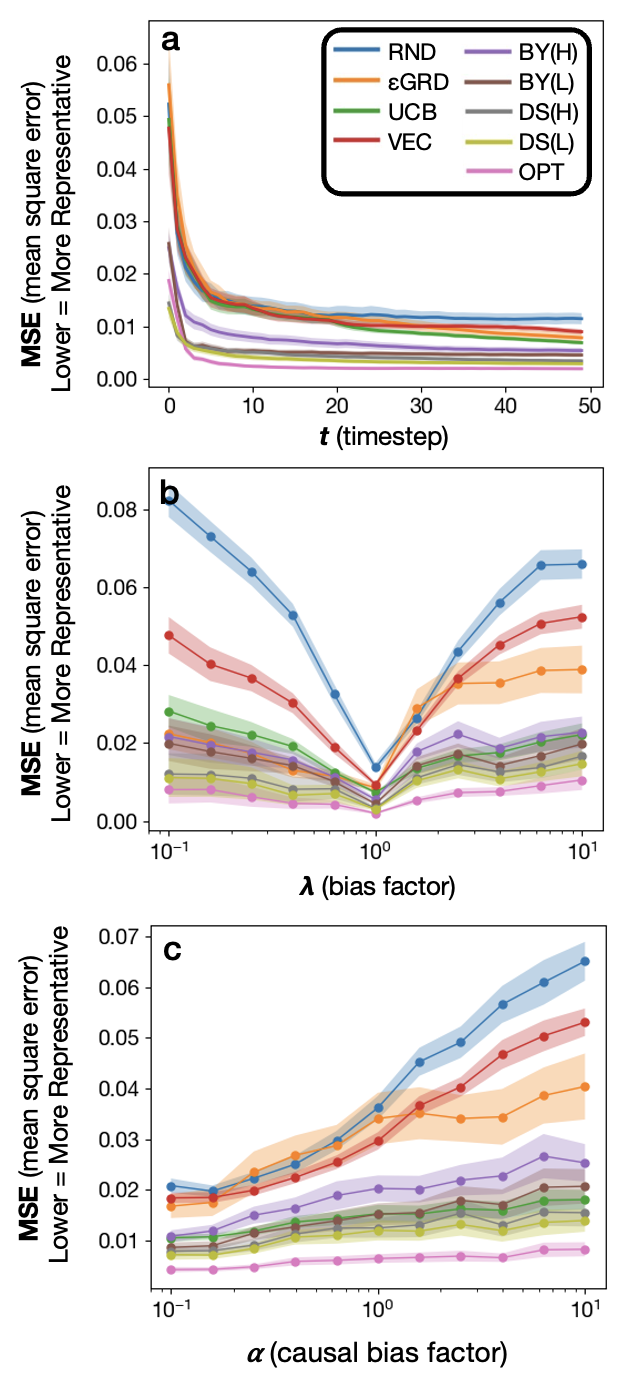}
    \caption{Dataset representativeness in the Intensive Care dataset measured by distance between cohort sensitive feature means and target vector $\boldsymbol{v} = \langle .5, \cdots, .5 \rangle$ as the cohort is constructed the no-bias case \textbf{(a)}, for the final cohort in the non-causal response bias case \textbf{(b)}, and for the final cohort in the causal distribution shift case \textbf{(c)}.
    Our proposed algorithms \textit{BY(H), BY(L), DS(H),} and \textit{DS(L)} outperform baseline sampling algorithms.
    Shaded regions indicate 95\% confidence intervals.}
    \label{fig:icu_sampling}
\end{figure}

Full results for all four datasets are available in appendix \ref{sup:sampling_other_datasets}.
In Figure \ref{fig:icu_sampling}a, we show the representativeness of the dataset constructed over time by each approach in a no-bias situation. 
While performance is similar between all algorithms, D-PBRS yields the most representative samples, often approaching fully informed OPT.
Response bias ($\lambda$) induces increased response rates of majority groups, i.e., individuals from majority groups are $\lambda$-times more likely to appear in a sample from biased sites compared to the group's true distribution at that site. 
Significant response bias in either direction harms the representativeness of the final cohort (\ref{fig:icu_sampling}b).
Yet, D-PBRS, and to a lesser extent PBRS, consistently yields more representative datasets than other sampling algorithms.
Figure \ref{fig:icu_sampling}c depicts dataset representativeness as a function of the casual bias ($\alpha$); as $\alpha$ increases, sampling a site increases the probability the member of majority groups will appear in future samples from that site. 
Similar to response bias, representativeness decreases as the bias becomes more pronounced.  
Unlike response bias, causal bias results in distribution shifts over time, increasing the difficult of accurately assessing which arm is best to sample.
Due to this shift, the advantage of PBRS and D-PBRS over other algorithms diminishes but is still present.

\subsection{Arm Sampling and Downstream Fairness}

\begin{figure*}[tbh!]
    \centering
    \includegraphics[width=1\linewidth]{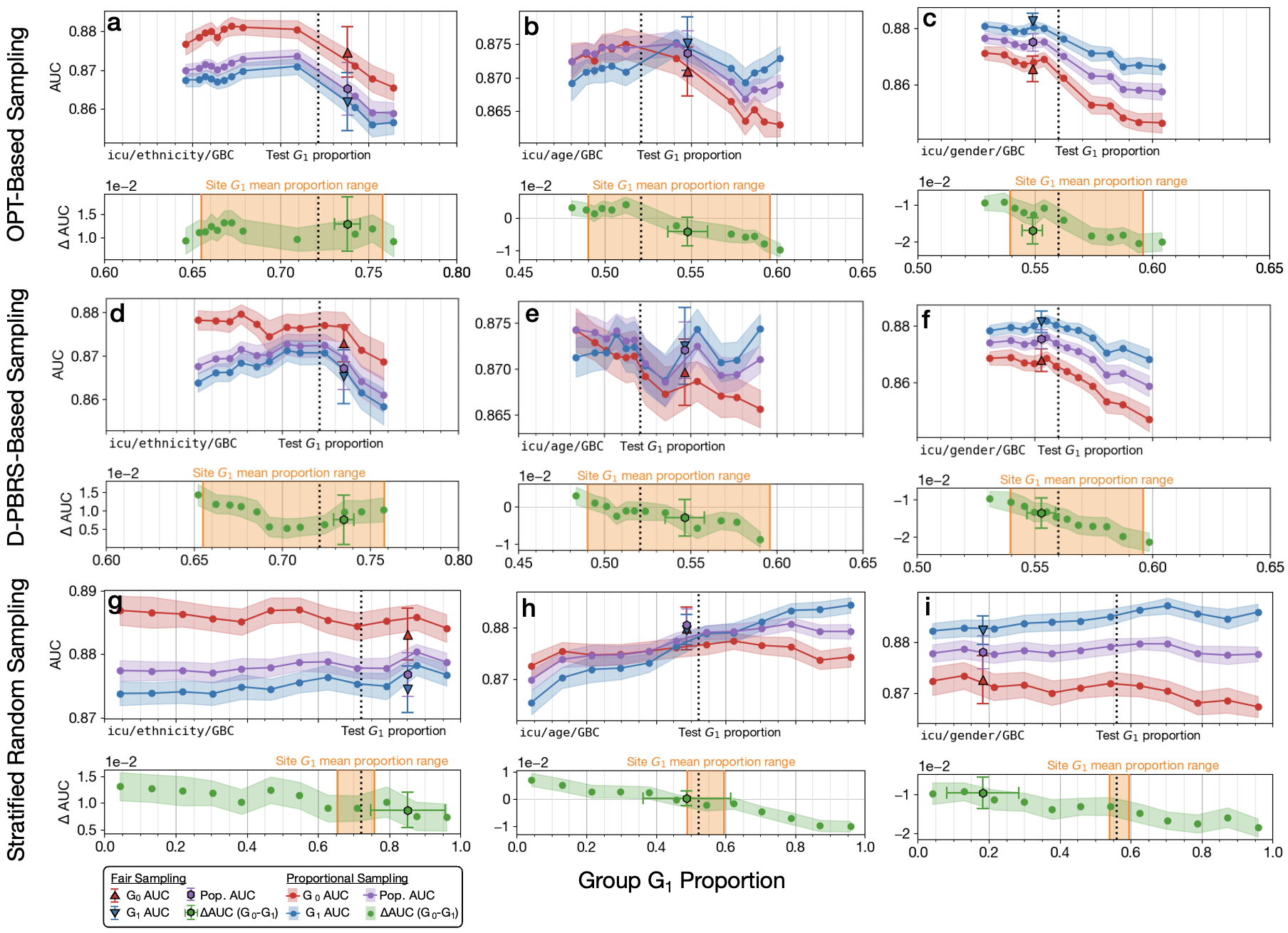}
    \caption{Population (purple) and subgroup (red and blue) AUCs for gradient-boosted classifiers in the Intensive Care dataset.
    Each column represents an analysis studying group proportions by one sensitive feature: \textbf{(a), (d), (g)} for ethnicity; \textbf{(b), (e), (h)} for age; and \textbf{(c), (f), (i)} for gender.
    Green points indicate the difference in subgroup AUCs (AUC$_{G_0} - $AUC$_{G_1}$). 
    Circles and shaded regions indicate quantile means and 95\% CIs for performance of representativeness-based samplers with varying $G_1$ proportions, while outlined triangles and hexagons with error bars indicate means and 95\% CIs for fairness-based samplers.
    The orange shading indicates the range of group $G_1$ proportions at each site.
    Subfigures \textbf{(a-c)} show classifier performance when training datasets are constructed by sampling arms with \textit{OPT}, subfigures \textbf{(d-f)} for sampling arms with \textit{D-PBRS}, and subfigures \textbf{(g-i)} for sampling directly from all training data to achieve a desired group proportion mix (stratified random sampling).}
    \label{fig:icu_univariate}
\end{figure*}

\begin{figure*}[tbh!]
    \centering
    \includegraphics[width=0.9\linewidth]{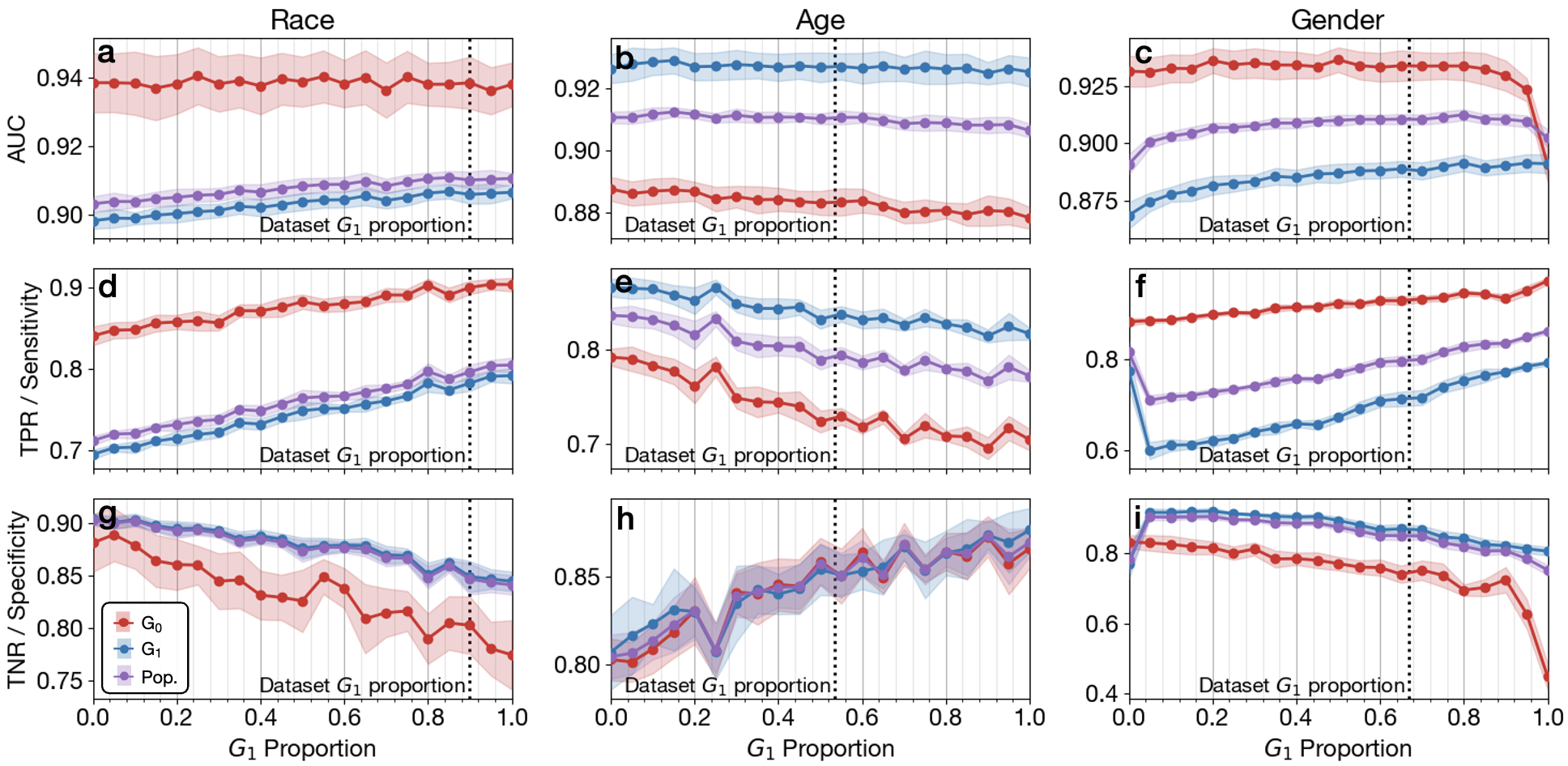}
    \caption{There is significant unfairness by race \textbf{(a, d, g)}, age \textbf{(b, e, h)}, and gender \textbf{(c, f, i)} in the Adult Income dataset.
    Population (purple) and subgroup (red and blue) AUCs \textbf{(a-c)}, TPRs \textbf{(d-f)}, TNRs \textbf{(g-i)} and 95\% CIs are plotted for varying $G_1$ proportions.}
    \label{fig:adl_fairness}
\end{figure*}

Over- and underrepresentation of particular groups in training data is a well known cause of unfairness within models trained on that data.
In figure \ref{fig:icu_univariate} we present population and group-specific test set AUC as a function of the $G_0/G_1$ split for each sensitive feature in the training dataset for our Intensive Care example. 
In figures \ref{fig:icu_univariate}a-c, the training dataset is constructed via arm based sampling using OPT to achieve the desired group proportions; in figures \ref{fig:icu_univariate}d-f, the training dataset is sampled from all available training data using stratified random sampling (SRS) to achieve the desired group proportions.
The outlined points throughout figures \ref{fig:icu_univariate}a-c and \ref{fig:icu_univariate}d-f indicate results for fair arm-based sampling and fair direct sampling, respectively.
Both variations of fair sampling achieve the desired goal of selecting a mix of groups $G_0$ and $G_1$ that minimizes performance difference between the two groups.
In the SRS case, this confirms the expected result that improving a group's proportion in training data will improve, or at least not hinder, that group's performance.
However, a group's performance improvement from increased representation can be quite limited at times.
Figure \ref{fig:icu_univariate}e shows how AUC increases for group $G_1$ as group $G_1$ proportion in the training data increases, but there is no significant concomitant decrease in group $G_0$ AUC.
Moreover, age (Fig. \ref{fig:icu_univariate}e) is the only sensitive feature for which there is some group $G_1$ proportion that equalizes AUC for groups $G_0$ and $G_1$ in the SRS analysis.
The SRS analyses of ethnicity and gender show consistently better classifier performance of groups $G_0$ and $G_1$, respectively, regardless of the training set proportions of these groups.
Thus, there must be additional factors affecting algorithmic fairness beyond group representation.
Given the theoretic results from the univariate case study in section 4, this is not unexpected if the noise values of the two groups are drastically different.

Another key result from this analysis is that the way datasets are constructed impacts the relationship between representation and algorithmic fairness.
The SRS results show the expected behavior: as $G_1$ proportion increases, $G_1$ test set AUC improves and $G_0$ test set AUC deteriorates,though the effects may not always be statistically significant.
On the other hand, arm-based sampling breaks this trend: when looking at both ethnicity and gender as sensitive features, increasing the $G_1$ training set proportion beyond its test set proportion causes deterioration of classifier performance for all groups.
Thus, attempting to achieve a desired group representation through adaptive sampling across multiple sites may yield unexpected downstream results.
We also note little difference between sampling with \textit{OPT} and \textit{D-PBRS} (Fig. \ref{fig:icu_univariate}), which indicates that the site-based framework, and not the representative sampling strategy, causes the discrepancy between SRS and arm-based methods.

\subsection{Fairness and Model Complexity}

\begin{figure}[htb!]
    \centering
    \includegraphics[width=0.6\linewidth]{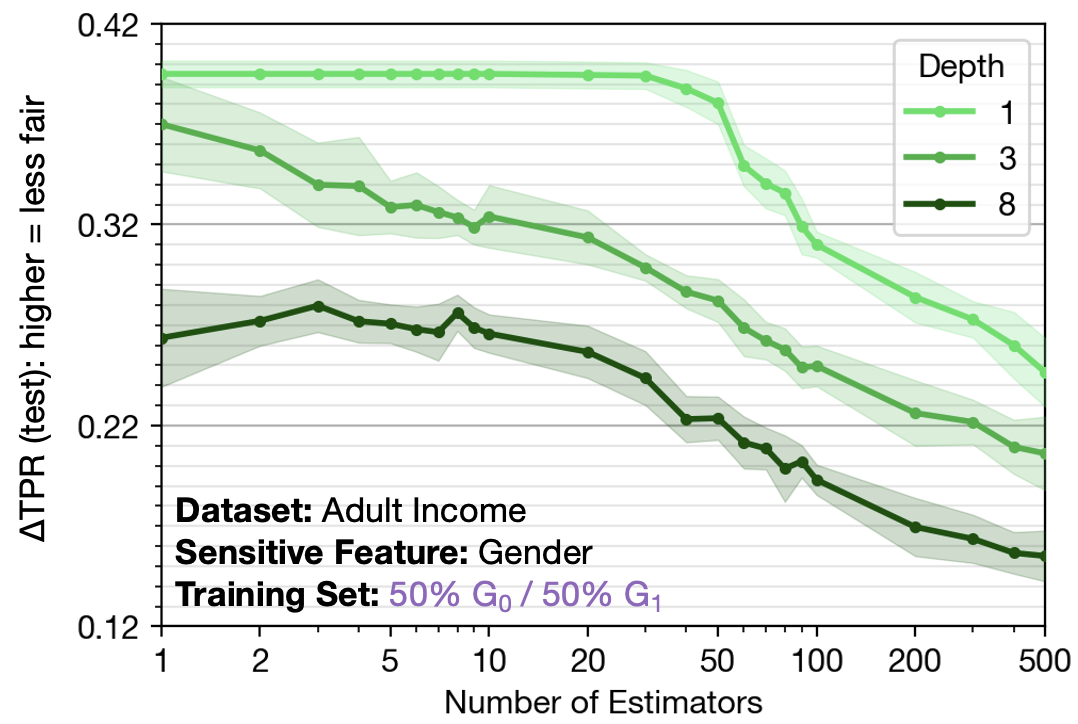}
    \caption{Increasing model complexity improves fairness by TPR parity for gender in the Adult Income dataset.
    Darker green lines indicate higher maximum tree depth for the GBC (higher complexity) and the x-axis shows number of estimation steps, with more indicating higher complexity.
    Shaded regions indicate 95\% CIs.}
    \label{fig:adl_gender_lineplot}
\end{figure}

\begin{figure*}[tbh!]
    \centering
    \includegraphics[width=1\linewidth]{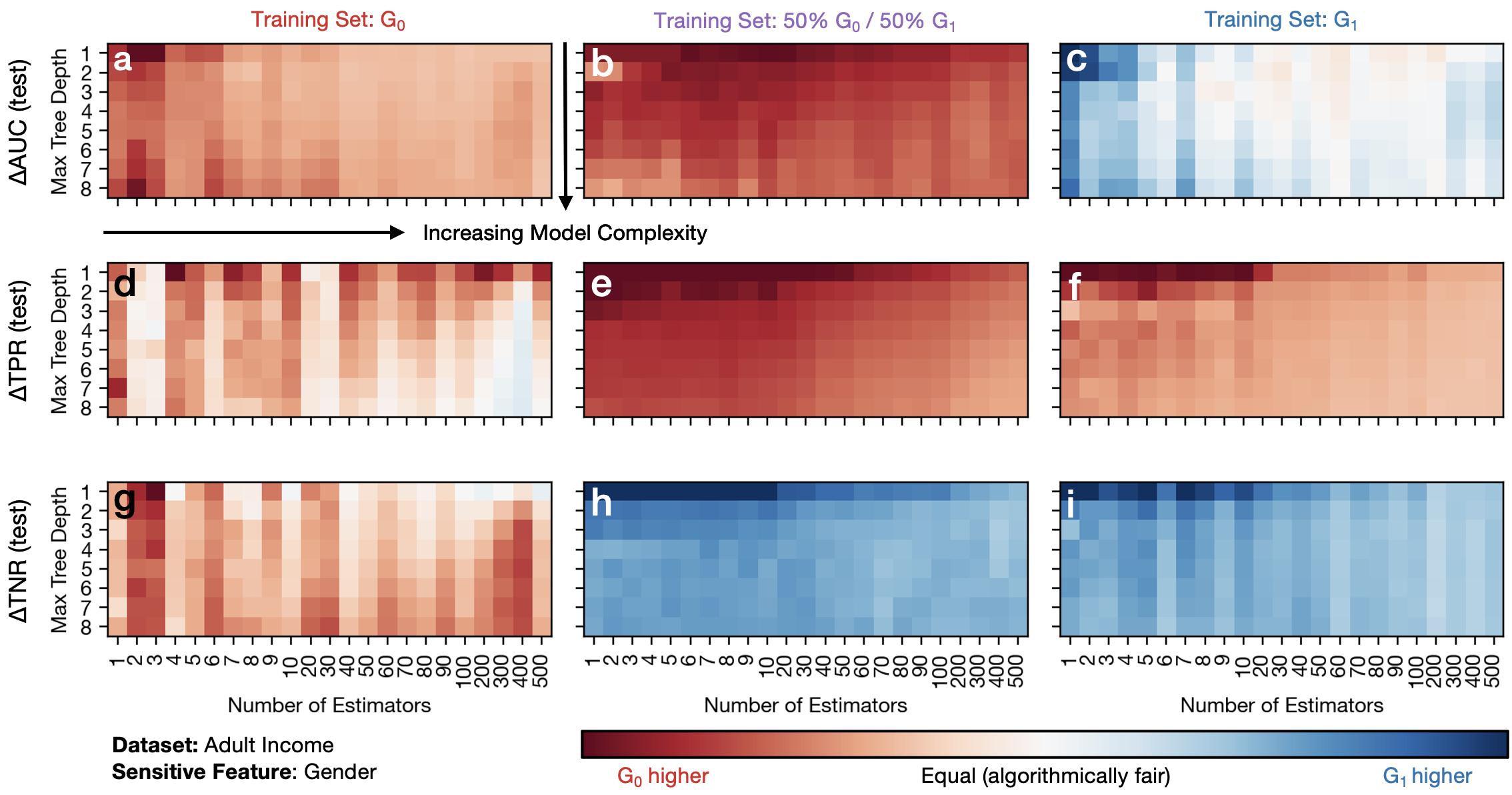}
    \caption{Increasing model complexity improves AUC \textbf{(a-c)}, TPR \textbf{(d-f)}, and TNR \textbf{(g-i)} parity for gradient boosted classifiers.
    Results are for the Adult Income dataset treating gender as the sensitive feature of interest.
    Darker red and blue colors indicate disparate performance favoring group $G_0$ and $G_1$, respectively, while paler colors indicate measure parity (fairness).
    Within each subfigure, rows represent maximum individual tree depths and columns indicate numbers of estimation steps.}
    \label{fig:adl_gender_complexity}
\end{figure*}

The Adult Income dataset shows significant AUC and TPR unfairness across all three tested sensitive features of race, age, and gender (Fig. \ref{fig:adl_fairness}).
Notably, modifying the training set proportions of groups $G_0$ and $G_1$ has limited effect on subgroup performance, except at the extremes (i.e., group proportions of 0 or 1).
Thus, this dataset highlights the practical case where modulating representation will not adequately address fairness concerns.
We shift our attention to increasing model complexity to better capture difficult-to-learn relationships between the features and labels.
We show how increasing complexity through greater tree depth and more estimation steps can reduce TPR unfairness between gender groups (Fig. \ref{fig:adl_gender_lineplot}).
A more complete complexity analysis shows similar results for AUC and TNR (Fig. \ref{fig:adl_gender_complexity}), and other sensitive features within the Adult Income dataset show similar patterns.
As tree depth and estimation steps increase, disparities in AUC, TPR, and TNR generally decrease, regardless of group representation in the training data.
Moreover, this decrease in unfairness through increased model complexity does not come at the expense of overall model performance.
In fact, classifier accuracy tends to \textit{improve} with increasing complexity (appendix \ref{sup:perf_complexity_analysis}, Fig. \ref{fig:adl_gender_auc}).
While the highest complexities --- estimators $\geq 200$ and depth $\geq 5$ --- show a moderate decrease in AUC, this is beyond the regions where we see the most substantial improvements in AUC unfairness.
We attribute these simultaneous improvements in both classifier accuracy and fairness with increased model complexity to the model being able to capture more complex data relationships.

\section{Discussion}
Representative datasets yield several benefits such as legitimacy, validity, equity, and generalizability.
In machine learning, generalizability is closely related to algorithmic fairness, a measure of prediction or performance parity between different groups.
In this paper, we analyze the relationship between representation and downstream algorithmic fairness in classification tasks across several datasets.
Contrary to our expectations, we find that more representative datasets rarely yield fairer classifiers.
Likewise, we find that datasets constructed to promote algorithmic fairness rarely are representative of the overall population.
We theorize that this tension between representativeness and fairness exists when groups differ significantly in their difficulty to learn.
If a large difficulty gap exists between groups, adding data points from the more difficult group may not be sufficient to overcome the disparity in classifier performance.
We show how an alternate approach, increasing model complexity, can help close this performance gap.
Thus, both representation and fairness may be simultaneously achieved.

In this paper, we also expand upon existing techniques for building a representative dataset from multiple data sources (e.g., multi-site clinical trial recruitment) through a Bayesian multi-armed bandit framework.
Our methods succeed at generating representative cohorts across a variety of biases and distributional shifts.
However, we find that downstream classifier performance differs significantly when cohorts are selected in a multi-site procedure to achieve a certain subgroup proportion compared to stratified random sampling of all records to achieve the same proportion.
The distribution of features, sensitive features, and labels over sites influences classifier fairness.
Thus, it is important to consider \textit{how} a dataset is constructed beyond its demographics matching a target distribution.

Despite the contributions of this work, there are some key limitations to note.
Representative sampling, as we have formulated it, focuses on matching a dataset's attribute means to a target population; however, the underlying distributions of the dataset and target population may differ substantially.
When it is important to match the shape of the dataset and target distributions, alternative measures for representation may perform better.
Moreover, it is important to consider what it means to match attribute means of a dataset to a ground truth population.
Such matching may be intuitive for physical or biological variables like age but becomes much more complicated for social variables like race, where the notion of ground truth does not necessarily apply.
Finally, it is important to note that increasing model complexity will not always substantially improve algorithmic fairness.
In fact, \cite{chen_why_2018} show that if a Bayesian optimal classifier is algorithmically unfair, further fairness cannot be enforced without loss of performance.
While we show that including additional data points from the disadvantaged group may not improve fairness, we echo their suggestion to collect additional features, if possible, in this situation.
Future work may include broader definitions of representation that are not group-centric, as well as expanding these results to additional definitions of fairness like procedural fairness as opposed to classifier parity measures.

We conclude that the relationship between dataset representativeness and downstream fairness is complicated and influenced by numerous factors.
While increasing a group's representation in a dataset sometimes improves that group's performance substantially, the practical constraints of dataset generation may sometimes cause the opposite effect.
Sometimes, changing a group's representation in a dataset has little impact on classifier performance; as shown, this may be due to learnability differences between groups.
In these cases, we suggest that one way of addressing this particular unfairness is to increase model complexity to more adequately capture complex data relationships.

\bibliography{bibliography}

\appendix

\section{Full Sampling Algorithms} \label{sup:algs}
\begin{algorithm}[tbh!]
    \small
    \SetAlgoLined
    \caption{Prior-based Bayesian Representative Sampling (PBRS) and Distributed PBRS (D-PBRS): sampling procedures for building a representative dataset.}
    \label{alg:sup_PBRS}
    \KwData{sites $\pazocal{S}$, budget $T$, target vector $\boldsymbol{v}$, prior mean and covariance 
    $\boldsymbol\theta_j, \boldsymbol \Sigma_j \forall j\in [m]$.}
    
    \KwResult{dataset $(\boldsymbol{X}, \boldsymbol{A}, \boldsymbol{Y})$.}
    
    $n_j \gets 0 \quad \forall j \in [m]$; \hfill // Number of times site $j$ is sampled
    
    $\boldsymbol{\Psi_j} = (n_j+1)\hat{\boldsymbol\Sigma_j}$; \hfill // Inverse scale matrix of normal inverse Wishart
    
    $\pazocal{W}_j^{-1}(\boldsymbol\theta_j, \boldsymbol{\Psi_j}, n_j)$; \hfill // Initialize normal inverse Wishart distribution
    
    $(\boldsymbol{X}, \boldsymbol{A}, \boldsymbol{Y}) = \emptyset$; \hfill // Initialize dataset
    
    \For{$t=0\ldots T$}{
        
        $\hat{\boldsymbol{\theta}_j}, \hat{\boldsymbol{\Sigma}_j} \sim \pazocal{W}_j^{-1}(\boldsymbol{\theta}_j, \boldsymbol{\Psi}_j, n_j),~ \forall j$;
        
        $\hat{\boldsymbol{a}_j} \sim \pazocal{N}(\hat{\boldsymbol{\theta}_j}, \hat{\boldsymbol{\Sigma}_j}),~\forall j$;
        
        $(\boldsymbol{X}, \boldsymbol{A}, \boldsymbol{Y}).\textrm{add}\big(\textrm{allocateAndSample}(\boldsymbol{A}, \boldsymbol{v}, \boldsymbol{a}, T, n)\big)$;\\

        $\textrm{updatePriors}\big( \boldsymbol{A}^{(t)}, \boldsymbol{\theta}_{j^*}, \boldsymbol{\Psi}_{j^*}, n_{j^*}, \beta, t\big) $;
        }

    $\textbf{return}\; \bigcup _{t=1} ^{T}(\boldsymbol{X}, \boldsymbol{A}, \boldsymbol{Y})^{(t)}$; \hfill // Final dataset \\
    \vspace{1em}
    \textbf{Function} allocateAndSample($\boldsymbol{A}$, $\boldsymbol{v}$, $\boldsymbol{a}$, $T$, $n$):
    
    \Indp
    
    // PBRS
    
    $j^* \gets \arg\min_{j}\mathbb{E}\bigg[ \pazocal{M}\bigg(\boldsymbol{v},  \big(\textrm{sum}(\boldsymbol{A})  + \boldsymbol{a}_j\big)/T\bigg) \bigg]$;\\
    // Arm with best improvement \\

    $(\boldsymbol{X}, \boldsymbol{A}, \boldsymbol{Y})^{(t)}\sim D_{j^*}$; \hfill // Sample data from arm $j^*$\\
    $n_{j^*} += 1$;\\
    $\textbf{return}\; (\boldsymbol{X}, \boldsymbol{A}, \boldsymbol{Y})^{(t)}$;
    \\\vspace{1em}
    \Indm

\textbf{Function} allocateAndSample($\boldsymbol{A}$, $\boldsymbol{v}$, $\boldsymbol{a}$, $T$, $n$):

    \Indp
    
    // Distributed PBRS
    
    $\rho_j \gets 0 \quad \forall j \in [m]$ \hfill // Resource vector
    
    $\rho^* \gets \arg\min_{\rho}\mathbb{E}\bigg[ \pazocal{M}\bigg(\boldsymbol{v},  ~\big(\textrm{sum}(\boldsymbol{A})  + \rho\boldsymbol{a}\big)/T\bigg) \bigg]$ \\ 
        \Indp subject to $\Sigma{\rho} = 1$

    \Indm
    
    $(\boldsymbol{X}, \boldsymbol{A}, \boldsymbol{Y})^{(t)} = \emptyset$
    
    \For{$j=0\ldots m$}{
            $(\boldsymbol{X}, \boldsymbol{A}, \boldsymbol{Y})^{(t, j)}\sim \lfloor\rho_j^*\rfloor D_{j}$; \hfill /* Sample from arm j a fraction of examples determined by $\rho^*$ for $j$ */ \\
            $(\boldsymbol{X}, \boldsymbol{A}, \boldsymbol{Y})^{(t)}.\textrm{add}\big((\boldsymbol{X}, \boldsymbol{A}, \boldsymbol{Y})^{(t, j)}\big)$;\\
            $n_j\; += \rho_j^*$;\\
        }
    $\textbf{return}\; (\boldsymbol{X}, \boldsymbol{A}, \boldsymbol{Y})^{(t)}$;

\end{algorithm}

\begin{algorithm}[tbh!]
\small
    \caption{Fair Arm-Based Sampling}
    \label{alg:sup_fair}
    \KwData{classifier $\pazocal{F}: \boldsymbol{X} \rightarrow \boldsymbol{Y}$, loss function $\mathcal{L}(\pazocal{F}(\boldsymbol{X}), \boldsymbol{Y} )$, validation data $(\boldsymbol{X}, \boldsymbol{A}, \boldsymbol{Y})'$}
    \textbf{Function} allocateAndSample($\bigcup_{t'=1}^{t-1}(\boldsymbol{X}, \boldsymbol{A},\boldsymbol{Y})^{(t')}$, $T$, $n$):
    
    \Indp
    train $\pazocal{F}$ using $\bigcup_{t'=1}^{t-1}(\boldsymbol{X}, \boldsymbol{A},\boldsymbol{Y})^{(t')}$ ;\\
    
    $g^* \gets \arg\max_{g} \mathbb{E} \bigg[ \mathcal{L}\big(\pazocal{F}\big( \boldsymbol{X}' \big), \boldsymbol{Y}' \big) \bigg| \boldsymbol{A}' = g \bigg]$ \\
    // Group with highest loss \\
    
    $j^* \gets \arg\max_{j} \sum_{t'=1}^T \textrm{count}(\boldsymbol{A}\; | \;\boldsymbol{A} = g^*, s = s_j)^{t'} $; \\
    // Site with the highest proportion of group $g^*$\\

    $(\boldsymbol{X}, \boldsymbol{A}, \boldsymbol{Y})^{(t)}\sim D_{j^*}$;

    $n_{j^*} += 1$;\\
    $\textbf{return}\; (\boldsymbol{X}, \boldsymbol{A}, \boldsymbol{Y})^{(t)}$;
\end{algorithm}

Algorithm \ref{alg:sup_PBRS} outlines the full procedures for PBRS and D-PBRS. 
Notably, the difference between these two algorithms is in the ``allocateAndSample'' function, which decides whether fractional allocation is allowed.

\section{Data Preprocessing and Computing Infrastructure}
For each dataset we follow a uniform procedure when preprocessing the raw data files. 
Ordinal features (e.g., an individual's income) are scaled between 0 and 1. 
Non-ordinal categorical features (e.g., an individual's occupation) are one-hot encoded. 
Binary features (e.g., ) are encoded as 0 and 1. 
All sensitive features are treated as binary or categorical. 
Only age and family income are non-categorical features in the raw datasets. In order to binarize these features we threshold on the mean age (family income) of the dataset and define categories of Young (Low Income) and Old (High Income). 
All analyses presented in this work were performed on an Apple M1 Max processor. Source code was written using Python 3.10.12.

\section{Proofs}\label{sup:proofs}

Provide full proofs, for the theoretical results presented in the main body.

\begin{proof}[Proof of Theorem \ref{thm:convex}]
    The objective in Equation \ref{eq:obj} is
    \[
        \min_{(\mathbf{X}, \mathbf{A}, \mathbf{Y})}\pazocal{M}\bigg(\mathbf{v},  ~\frac{1}{|\mathbf{A}|}\sum_{\mathbf{a}\in \mathbf{A}}\mathbf{a} \bigg)
    \]
    and the objective in Equation \ref{eq:obj2} is
    \[
         \min_{\mathbf{A}} ~\pazocal{M}\bigg(\mathbf{v},  ~\frac{1}{T}\sum_{t=1}^T\textrm{avg}\big(\mathbf{A}^{(t)}\big)\bigg)
    \]
    To first prove equivalence between these two objectives when each sample yields $k$ individuals, we restate the derivation provided in the main
    \[
        \frac{1}{|\mathbf{A}|}\sum_{\mathbf{a}\in \mathbf{A}}\mathbf{a} = \frac{1}{T k}\sum_{\mathbf{a}\in \mathbf{A}}\mathbf{a} = \frac{1}{T}\sum_{t=1}^T \sum_{\mathbf{a}\in \mathbf{A}^{(t)}}\frac{\mathbf{a}}{k} = \frac{1}{T}\sum_{t=1}^T \textrm{avg}\big(\mathbf{A}^{(t)}\big)
    \]
    as such, we see that for any $\mathbf{A} = \cup_{t=1}^T \mathbf{A}^{(t)}$, 
    \[
        \pazocal{M}\big(\mathbf{v}, \frac{1}{|\mathbf{A}|}\sum_{\mathbf{a}\in \mathbf{A}}\mathbf{a}\big) = \pazocal{M}(\mathbf{v}, \frac{1}{T}\sum_{t=1}^T \textrm{avg}\big(\mathbf{A}^{(t)} \big)\big),
    \]
    and the two objectives have equal optimums.

    To show the convexity of the data collector's objective w.r.t. the samples $\mathbf{A}^{(t)}$, we note that $\pazocal{M}(\mathbf{v}, \mathbf{u})$, is convex in $\mathbf{u}\in\mathbb{R}^d$, and thus for any linear function $f$, the composition  $\pazocal{M}\big(\mathbf{v}, f(\mathbf{u})\big)$ is also convex in $\mathbf{u}$.
   The function $\frac{1}{T}\sum_{t=1}^{T} \textrm{avg}\big(\mathbf{A}^{(t)}\big)$ is linear in the collection of samples $\mathbf{A}^{(T)} = \cup_{t=1}^{(T)}\mathbf{A}^{(t)}$.
    Thus, $\pazocal{M}$ is convex in the samples $\mathbf{A}^{(T)}$.
\end{proof}

\begin{proof}[Proof of Theorem \ref{thm:unfair_1}]
    Let $(x, y)$ be one datapoint, i.e., a feature and label respectively. Suppose that for a given $x$ the label $Y$ is induced via $y = \mathbb{I}\big[ x + \varepsilon_g \geq \theta_g \big]$ where $\varepsilon_g \sim \mathcal{N}(\mu_g, \sigma_g)$. 
    Let $(\mathbf{X}, \mathbf{Y})$ be a dataset of $n_0$ such examples from group $g_0$ and $n_1$ such examples from group $g_1$.
    Let $\pazocal{F}$ be the classifier with the highest accuracy on $(\mathbf{X}, \mathbf{Y})$.
    
    Then the expected unfairness of $\pazocal{F}$ with respect to each group's true distribution over features and labels $D_g$, can be written as 
    \[
        \delta = \textrm{error}(\pazocal{F}, D_0) - \textrm{error}(\pazocal{F}, D_0)
    \]
    In the case that $\pazocal{F}$ is a threshold classifier acting on both groups, the classifier with the highest accuracy on data $(\mathbf{X}, \mathbf{Y})$ will have the propriety that
    \[
        \pazocal{F}(x|g) =
            1 ~\textrm{ if } ~x \geq 1/n_g\sum_{x_j\in \mathbf{X} | g} x, ~\quad\textrm{ and }\\
            0 ~\textrm{ otherwise}
    \]
    where $1/n_g\sum_{x_j\in \mathbf{X} | g} x$ is the mean value of all features in $\mathbf{X}$ which correspond to group $g$.
    Thus, each error term $\textrm{error}(\pazocal{F}|g)$ is proportional to the empirical mean $1/n_g\sum_{x_j\in \mathbf{X} | g} x$ and the true feature mean $\theta_g$. 
    By the Mean Absolute Difference for normal distributions, this value is 
    \[
        \sigma_g\sqrt{2/\pi n_g}
    \]
    for each group.
    Thus the expected difference in error rates is
    $\mathbb{E}\big[ \delta \big] = \sqrt{2/\pi}\big(\sigma_0\sqrt{1/n_0} - \sigma_1\sqrt{1/n_1}\big)$

\end{proof}

\begin{proof}[Proof of Theorem \ref{thm:unfair_2}]
    By Theorem \ref{thm:unfair_1}, having an unfairness of size at most $\delta$ requires 
    \[
    -\delta \leq\sqrt{2/\pi}\big(\sigma_0\sqrt{1/n_0} - \sigma_0\sqrt{1/n_0}\big) \leq \delta
    \]
    By first examining the left-side inequality with respect to group $g_1$ we get
    \[
    \sigma_1 \sqrt{1/n_1} \leq \delta \sqrt{\pi/2} + \sigma_0 \sqrt{1/n_0}
    \]
    \[
        \Rightarrow 1/n_1 \leq \bigg( \frac{\delta \sqrt{\pi/2} + \sigma_0 \sqrt{1/n_0}}{\sigma_1}\bigg)^2
    \]
    \[
        \Rightarrow n1 \geq \bigg( \frac{\sigma_1}{\delta\sqrt{\pi/2n_0 + \sigma_0}} \bigg)^2n_0
    \]
    A similar algebraic reduction when examining the right-side inequality with respect to group $g_0$ yields the other inequality.
\end{proof}

\section{Experimental Results}\label{sup:experimental}
\subsection{Sampling in Other Datasets}\label{sup:sampling_other_datasets}

\begin{figure}[htb!]
    \centering
    \includegraphics[width=0.4\linewidth]{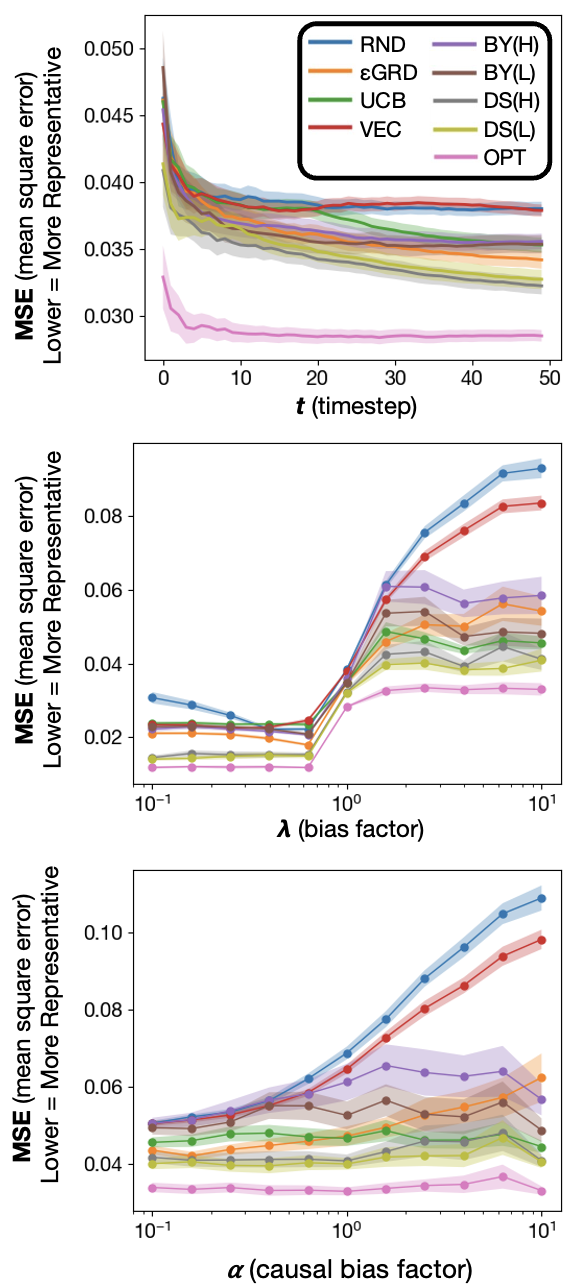}
    \caption{Dataset representativeness for the no-bias, non-causal response bias, and causal response bias cases in the Law School dataset.}
    \label{fig:sup_law_sampling}
\end{figure}

\begin{figure}[htb!]
    \centering
    \includegraphics[width=0.4\linewidth]{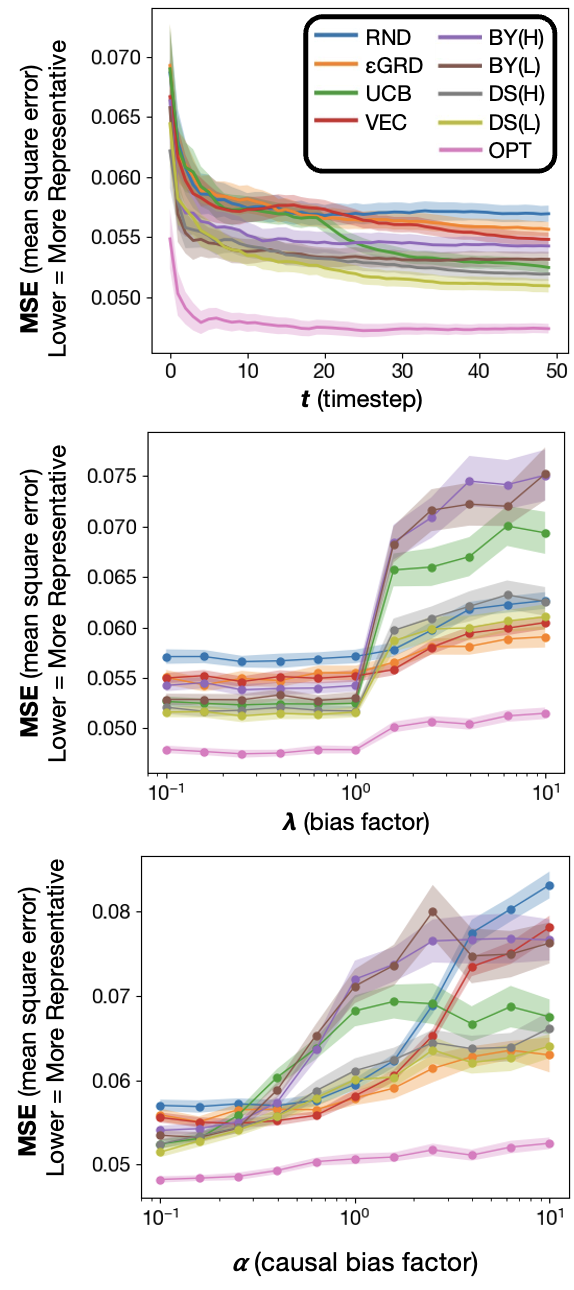}
    \caption{Dataset representativeness for the no-bias, non-causal response bias, and causal response bias cases in the Lending Club dataset.}
    \label{fig:sup_lnd_sampling}
\end{figure}

\begin{figure}[htb!]
    \centering
    \includegraphics[width=0.4\linewidth]{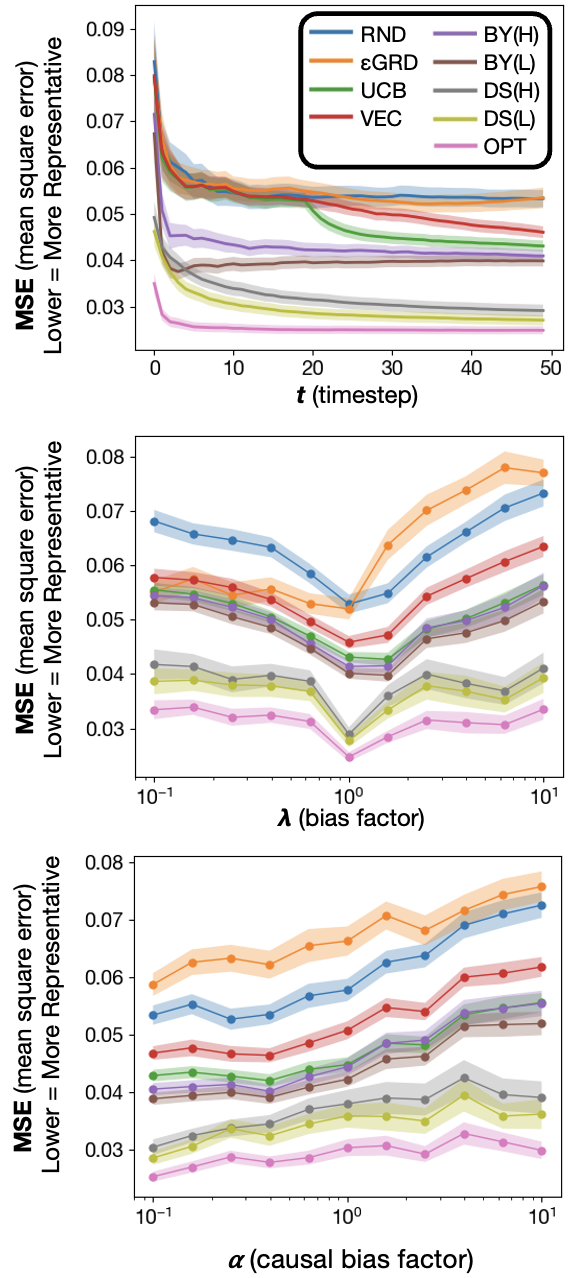}
    \caption{Dataset representativeness for the no-bias, non-causal response bias, and causal response bias cases in the Texas Salary dataset.}
    \label{fig:sup_txs_sampling}
\end{figure}

As an extension of main paper figure \ref{fig:icu_sampling}, we show performance of the nine algorithms on all four tested datasets in figures \ref{fig:sup_law_sampling}, \ref{fig:sup_lnd_sampling}, and \ref{fig:sup_txs_sampling}. In all cases, the fully-informed algorithm OPT achieves the best performance, typically followed by D-PBRS, then PBRS and UCB-LCB.

\subsubsection{Response Bias}
Recall that response bias is defined by parameters $\lambda$ the increased probability of majority group members responding in a sample, and $\gamma$ the number of sites which have $\lambda$-bias. 
We can convert $\lambda$ to a proportion scaling factor $b$ through the transformation $b = \frac{\lambda}{1 + \lambda}$.
To implement response bias for binary sensitive features $\pazocal{A}=\{0, 1\}^d$, we choose $1$ to represent the larger group. 
For example if there are two features, age (Old or Young) and gender (Male or Female), where 70\% of individuals are Old and 60\% are Female, then $\langle 1, 1 \rangle$ corresponds to an individual who is both Old and Female. 
For example if there are two features, age (Old or Young) and gender (Male or Female), where 70\% of individuals are Old and 60\% are Female, then $\langle 1, 1 \rangle$ corresponds to an individual who is both Old and Female. 
When sampling from site $j$, rather than selecting $k$ examples uniformly at random from the associated data partition, $k$ examples are selected randomly with weights proportional to $\sum_{\ell=1}^d\big(b\cdot a_{\ell} + (1-b)\cdot(1-a_{\ell})\big)^2$. 
Thus an individual with features in each majority group (i.e., $\boldsymbol{a}=\mathbf{1}$) has $d \times \lambda^2$ times more sample weight than an individual with features from each minority group (i.e., $\boldsymbol{a}=\mathbf{0}$). When $\lambda=1$, then $b = 0.5$ and this sum reduces to $0.5^2$ for \emph{all} individuals and the no-bias setting is recovered.

\subsubsection{Causal Distribution Shift Bias}
Recall that casual distribution shifts are defined by parameter $\alpha$ where the response probability $p$ of each individual at site $j$ is scaled by $p_\textrm{post} = p_\textrm{pre}^{1 + \alpha \times \rho_j}$ when sampled.
To implement this for binary groups, we again represent each majority group with value $1$ and minority groups with value $0$.
Similar to the case of response bias, we re-weight the sample probabilities of the data partition associated with each site.
The sample probabilities for each individual at site $j$, after $n_j$ sampling iterations, is proportional to $\prod^{\sum^{n_j}{\rho_j}} p^{1 + \alpha \times \rho_j}$, where $p$ is the initial response probability of the individual, determined as described in the response bias section above.
As $\alpha$ or $\lambda$ increase, members from minority groups are less likely to appear in repeated samples from the same site.

\subsection{Arm Sampling and Downstream Fairness in Other Datasets}
\begin{figure*}[tbh]
    \centering
    \includegraphics[width=1\linewidth]{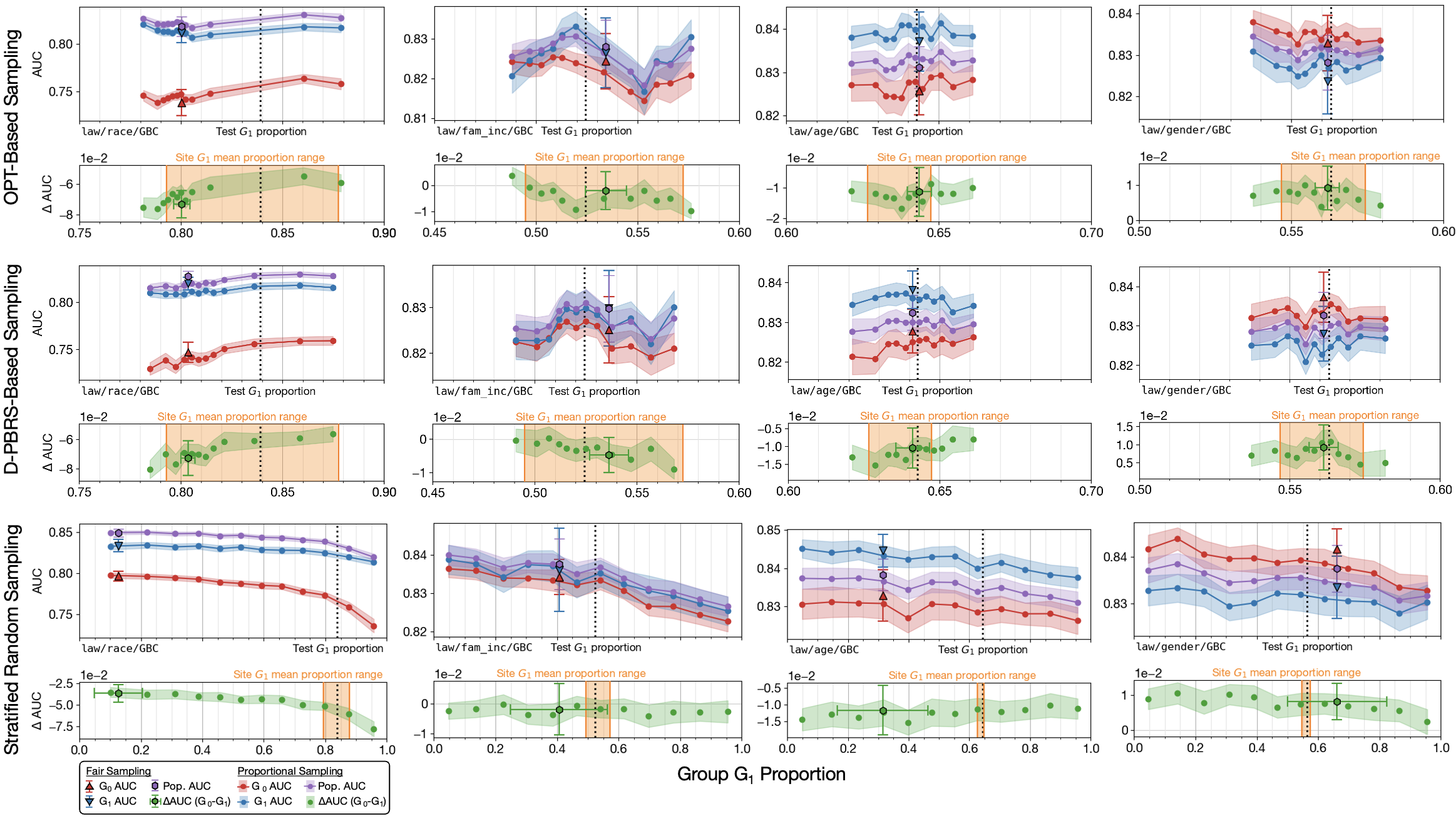}
    \caption{Population (purple) and subgroup (red and blue) AUCs for gradient-boosted classifiers in the Law School dataset.
    Each column represents an analysis studying group proportions by one sensitive feature.
    Green points indicate the difference in subgroup AUCs (AUC$_{G_0} - $AUC$_{G_1}$). 
    Circles and shaded regions indicate quantile means and 95\% CIs for performance of representativeness-based samplers with varying $G_1$ proportions, while outlined triangles and hexagons with error bars indicate means and 95\% CIs for fairness-based samplers.
    The orange shading indicates the range of group $G_1$ proportions at each site.
    The top row shows classifier performance when training datasets are constructed by sampling arms with \textit{OPT}, the middle row for sampling arms with \textit{D-PBRS}, and the bottom row for SRS.}
    \label{fig:law_arm_fairness_GBC}
\end{figure*}

\begin{figure*}[tbh]
    \centering
    \includegraphics[width=1\linewidth]{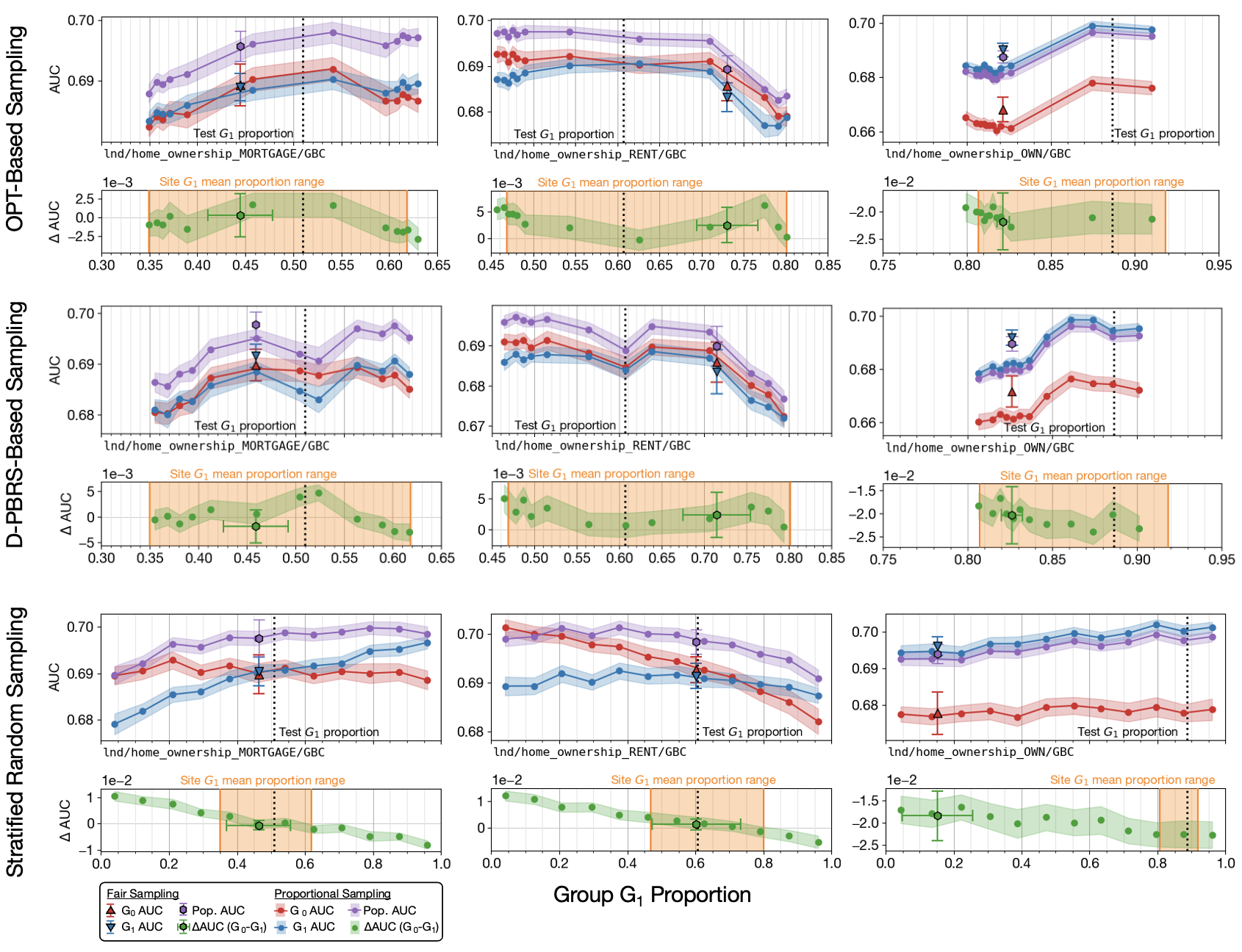}
    \caption{Population (purple) and subgroup (red and blue) AUCs for gradient-boosted classifiers in the Law School dataset.
    Each column represents an analysis studying group proportions by one sensitive feature.
    Green points indicate the difference in subgroup AUCs (AUC$_{G_0} - $AUC$_{G_1}$). 
    Circles and shaded regions indicate quantile means and 95\% CIs for performance of representativeness-based samplers with varying $G_1$ proportions, while outlined triangles and hexagons with error bars indicate means and 95\% CIs for fairness-based samplers.
    The orange shading indicates the range of group $G_1$ proportions at each site.
    The top row shows classifier performance when training datasets are constructed by sampling arms with \textit{OPT}, the middle row for sampling arms with \textit{D-PBRS}, and the bottom row for SRS.}
    \label{fig:lnd_arm_fairness_GBC}
\end{figure*}

\begin{figure*}[tbh]
    \centering
    \includegraphics[width=1\linewidth]{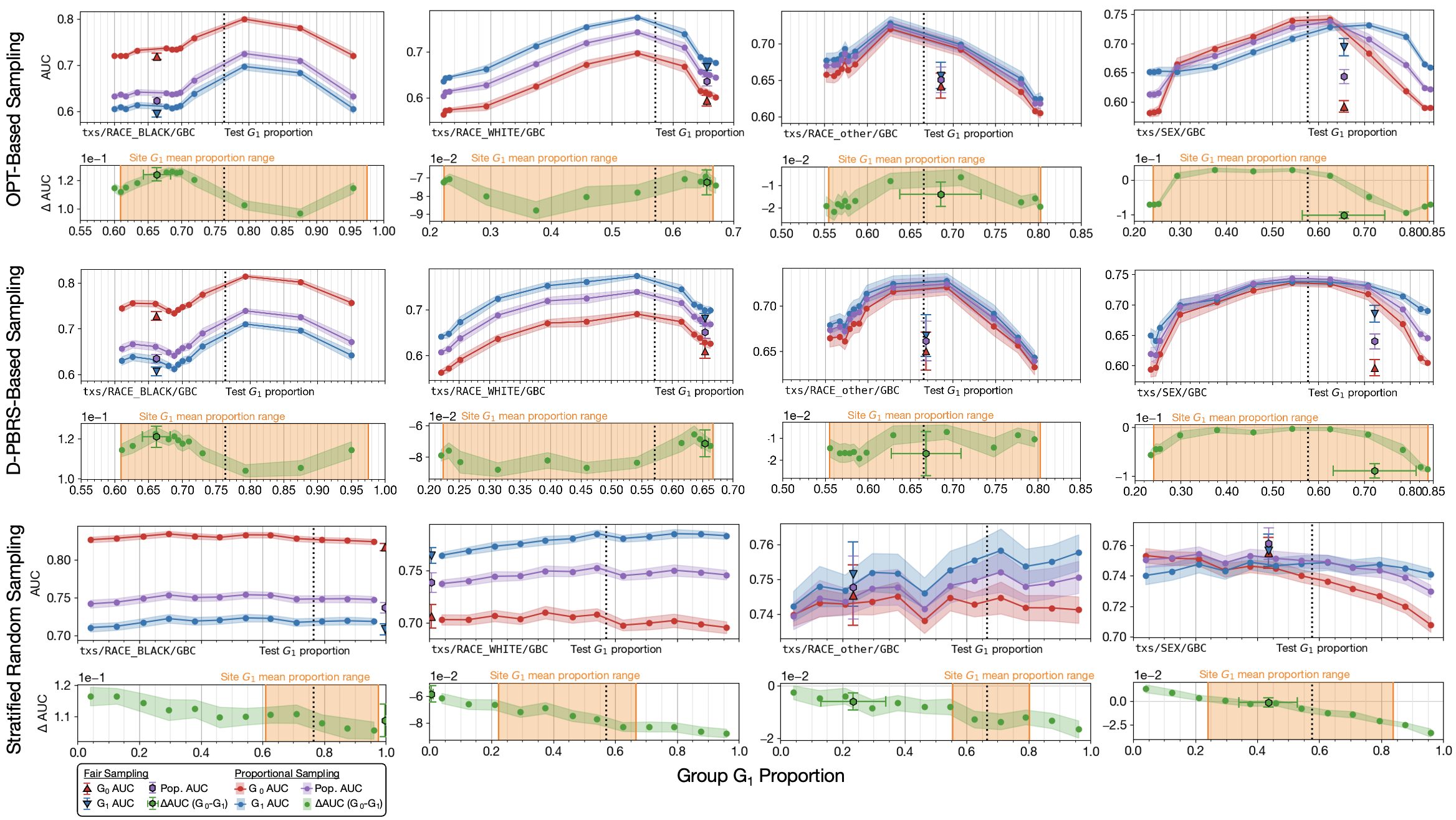}
    \caption{Population (purple) and subgroup (red and blue) AUCs for gradient-boosted classifiers in the Texas Salary dataset.
    Each column represents an analysis studying group proportions by one sensitive feature.
    Green points indicate the difference in subgroup AUCs (AUC$_{G_0} - $AUC$_{G_1}$). 
    Circles and shaded regions indicate quantile means and 95\% CIs for performance of representativeness-based samplers with varying $G_1$ proportions, while outlined triangles and hexagons with error bars indicate means and 95\% CIs for fairness-based samplers.
    The orange shading indicates the range of group $G_1$ proportions at each site.
    The top row shows classifier performance when training datasets are constructed by sampling arms with \textit{OPT}, the middle row for sampling arms with \textit{D-PBRS}, and the bottom row for SRS.}
    \label{fig:txs_arm_fairness_GBC}
\end{figure*}

\begin{figure*}[tbh]
    \centering
    \includegraphics[width=1\linewidth]{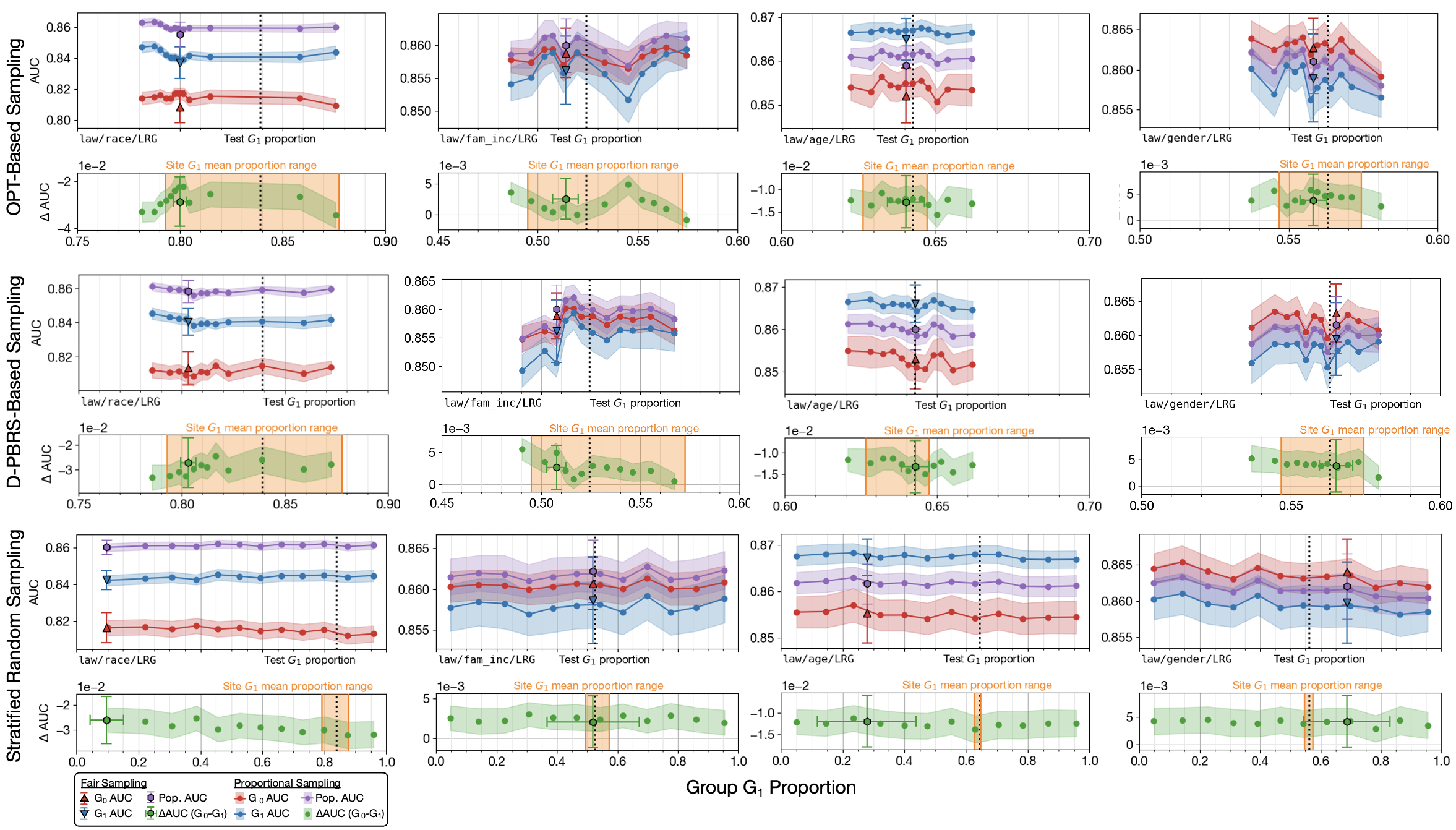}
    \caption{Population (purple) and subgroup (red and blue) AUCs for logistic regression in the Law School dataset.
    Each column represents an analysis studying group proportions by one sensitive feature.
    Green points indicate the difference in subgroup AUCs (AUC$_{G_0} - $AUC$_{G_1}$). 
    Circles and shaded regions indicate quantile means and 95\% CIs for performance of representativeness-based samplers with varying $G_1$ proportions, while outlined triangles and hexagons with error bars indicate means and 95\% CIs for fairness-based samplers.
    The orange shading indicates the range of group $G_1$ proportions at each site.
    The top row shows classifier performance when training datasets are constructed by sampling arms with \textit{OPT}, the middle row for sampling arms with \textit{D-PBRS}, and the bottom row for SRS.}
    \label{fig:law_arm_fairness_LRG}
\end{figure*}

\begin{figure*}[tbh]
    \centering
    \includegraphics[width=1\linewidth]{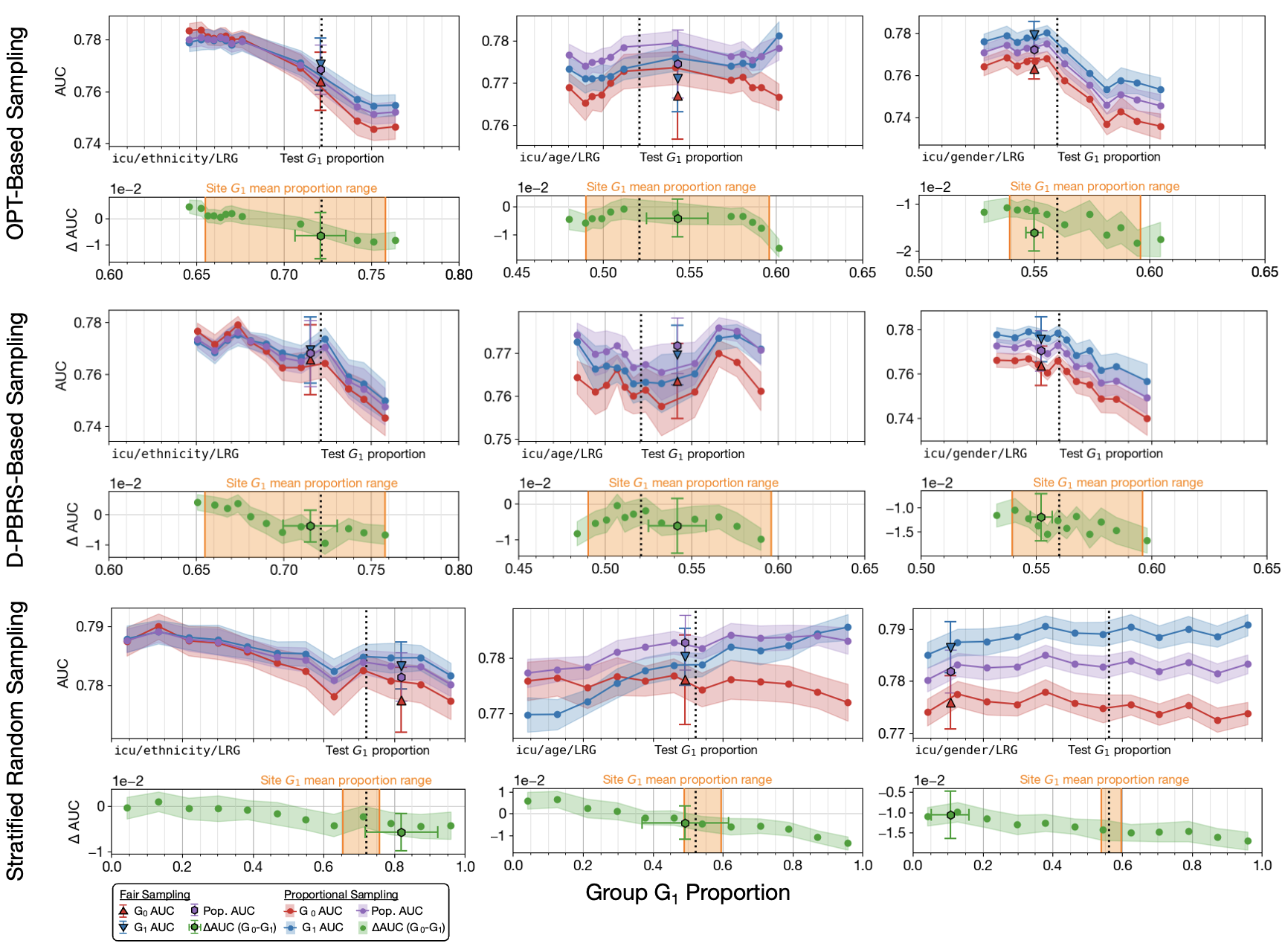}
    \caption{Population (purple) and subgroup (red and blue) AUCs for logistic regression in the Intensive Care dataset.
    Each column represents an analysis studying group proportions by one sensitive feature.
    Green points indicate the difference in subgroup AUCs (AUC$_{G_0} - $AUC$_{G_1}$). 
    Circles and shaded regions indicate quantile means and 95\% CIs for performance of representativeness-based samplers with varying $G_1$ proportions, while outlined triangles and hexagons with error bars indicate means and 95\% CIs for fairness-based samplers.
    The orange shading indicates the range of group $G_1$ proportions at each site.
    The top row shows classifier performance when training datasets are constructed by sampling arms with \textit{OPT}, the middle row for sampling arms with \textit{D-PBRS}, and the bottom row for SRS.}
    \label{fig:icu_arm_fairness_LRG}
\end{figure*}

\begin{figure*}[tbh]
    \centering
    \includegraphics[width=1\linewidth]{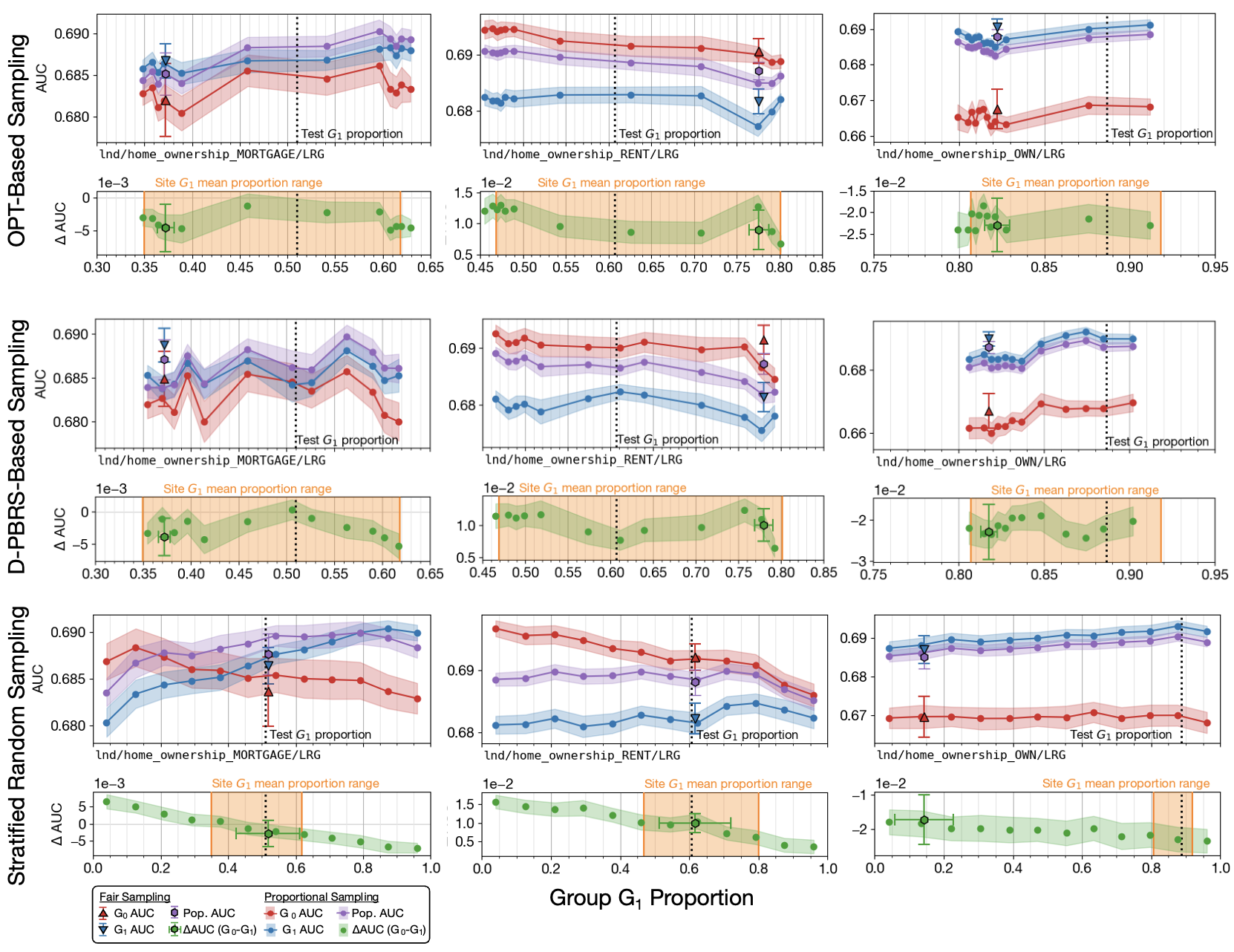}
    \caption{Population (purple) and subgroup (red and blue) AUCs for logistic regression in the Lending Club dataset.
    Each column represents an analysis studying group proportions by one sensitive feature.
    Green points indicate the difference in subgroup AUCs (AUC$_{G_0} - $AUC$_{G_1}$). 
    Circles and shaded regions indicate quantile means and 95\% CIs for performance of representativeness-based samplers with varying $G_1$ proportions, while outlined triangles and hexagons with error bars indicate means and 95\% CIs for fairness-based samplers.
    The orange shading indicates the range of group $G_1$ proportions at each site.
    The top row shows classifier performance when training datasets are constructed by sampling arms with \textit{OPT}, the middle row for sampling arms with \textit{D-PBRS}, and the bottom row for SRS.}
    \label{fig:lnd_arm_fairness_LRG}
\end{figure*}

\begin{figure*}[tbh]
    \centering
    \includegraphics[width=1\linewidth]{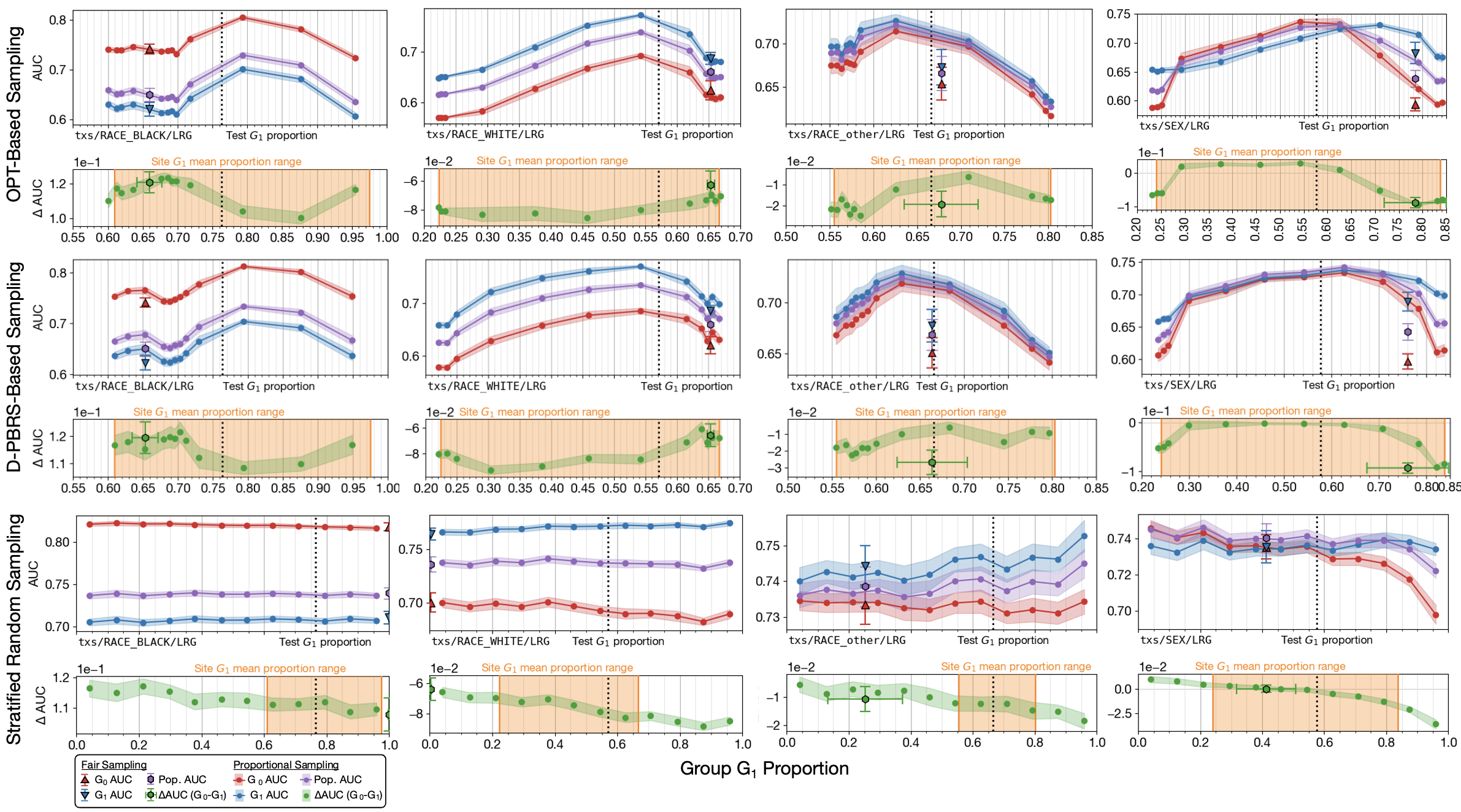}
    \caption{Population (purple) and subgroup (red and blue) AUCs for logistic regression in the Texas Salary dataset.
    Each column represents an analysis studying group proportions by one sensitive feature.
    Green points indicate the difference in subgroup AUCs (AUC$_{G_0} - $AUC$_{G_1}$). 
    Circles and shaded regions indicate quantile means and 95\% CIs for performance of representativeness-based samplers with varying $G_1$ proportions, while outlined triangles and hexagons with error bars indicate means and 95\% CIs for fairness-based samplers.
    The orange shading indicates the range of group $G_1$ proportions at each site.
    The top row shows classifier performance when training datasets are constructed by sampling arms with \textit{OPT}, the middle row for sampling arms with \textit{D-PBRS}, and the bottom row for SRS.}
    \label{fig:txs_arm_fairness_LRG}
\end{figure*}

We present analyses for the population and group-wise accuracy of classifiers trained on datasets which vary in proportion of each sensitive feature for the arm sampling data domains not included in the main body (Law School, Lending Club, and Texas Salary).
Figures \ref{fig:law_arm_fairness_GBC}-\ref{fig:txs_arm_fairness_LRG} show population, and group-wise, performance as a function the fraction of samples in the training data which are from $G_1$ (shown on each plot).
For each dataset we present two analyses: one with a gradient boosted classifier (GBC) and one with a logistic regression (LRG) classifier.
Each figure shows three sampling strategies: \textit{OPT}, \textit{D-PBRS}, and SRS, paralleling the methods and results for figure \ref{fig:icu_univariate} from the main body.
As discussed in the main body, there are two key observations in these figures.
First, an increase in the representation of a given group does not always significantly improve downstream performance on that group even in the SRS case, e.g., Black race in the Texas Salary dataset (Fig. \ref{fig:txs_arm_fairness_GBC} bottom left).
However, in other cases improved representation results in better performance for that group, e.g., Sex $G_0$ in the Texas Salary dataset (Fig. \ref{fig:txs_arm_fairness_GBC} bottom right).
Second, the sampling method plays a crucial role in the relationship between downstream fairness and representation.
The arm-based sampling methods \textit{OPT} and \textit{D-PBRS} often show very different subgroup performance than SRS (Fig. \ref{fig:txs_arm_fairness_GBC})
Overall, results of the analyses with logistic regression exhibit similar results patterns to those for gradient boosted classifiers.

\subsection{Fairness and Complexity in Other Datasets}

\subsubsection{Baseline Unfairness}
\begin{figure*}[tbh]
    \centering
    \includegraphics[width=1\linewidth]{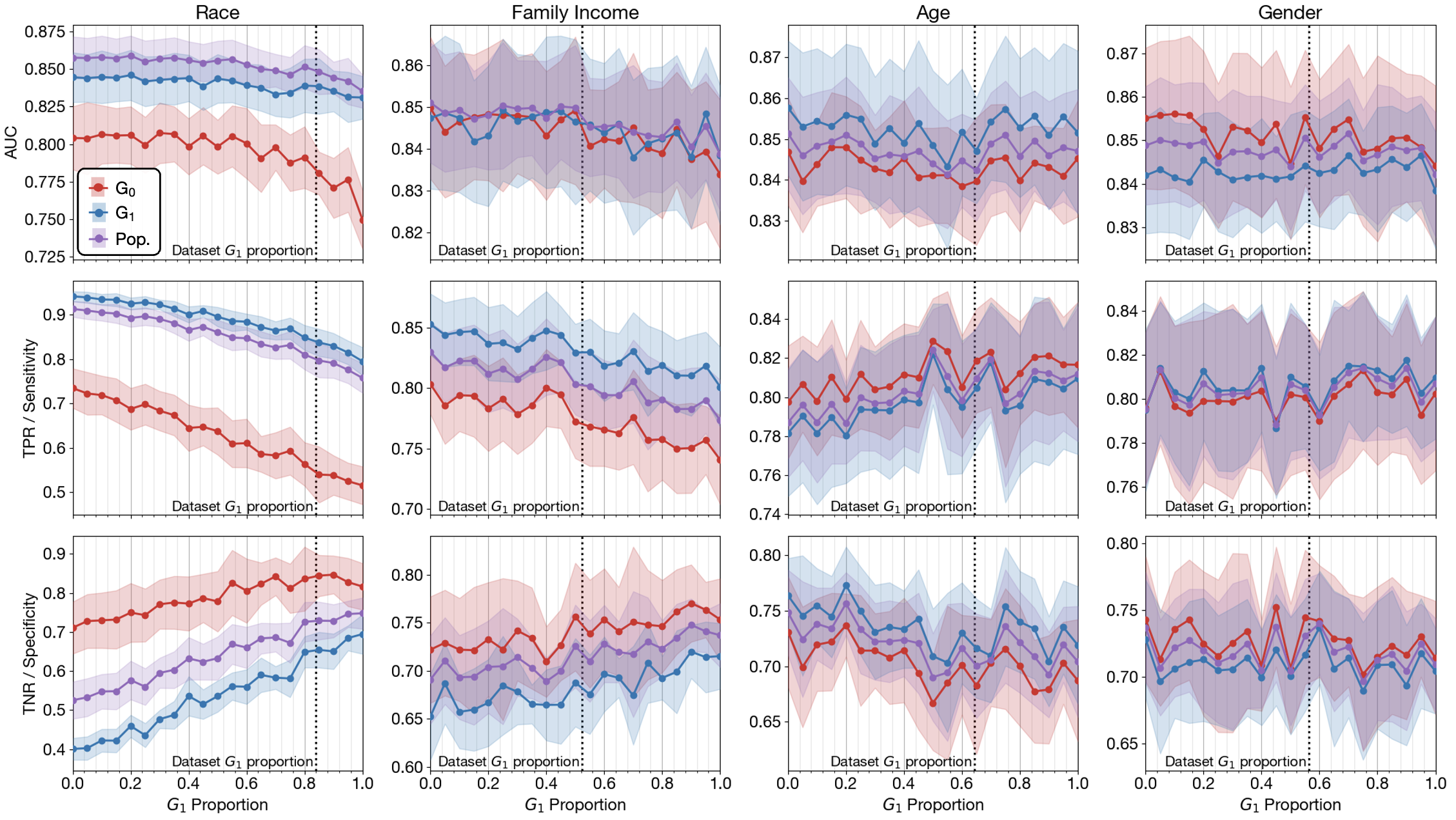}
    \caption{Population (purple) and subgroup (red and blue) AUCs, TPRs, and TNRs for gradient boosted classifiers in the non-arm-based Law School dataset.
    Circles and shaded regions indicate quantile means and 95\% CIs for performance of representativeness-based samplers with varying $G_1$ proportions.
    Each column represents an analysis studying group proportions by one sensitive feature while each row indicates a different classifier performance measure.}
    \label{fig:law_non_arm_fairness}
\end{figure*}

\begin{figure*}[tbh]
    \centering
    \includegraphics[width=0.35\linewidth]{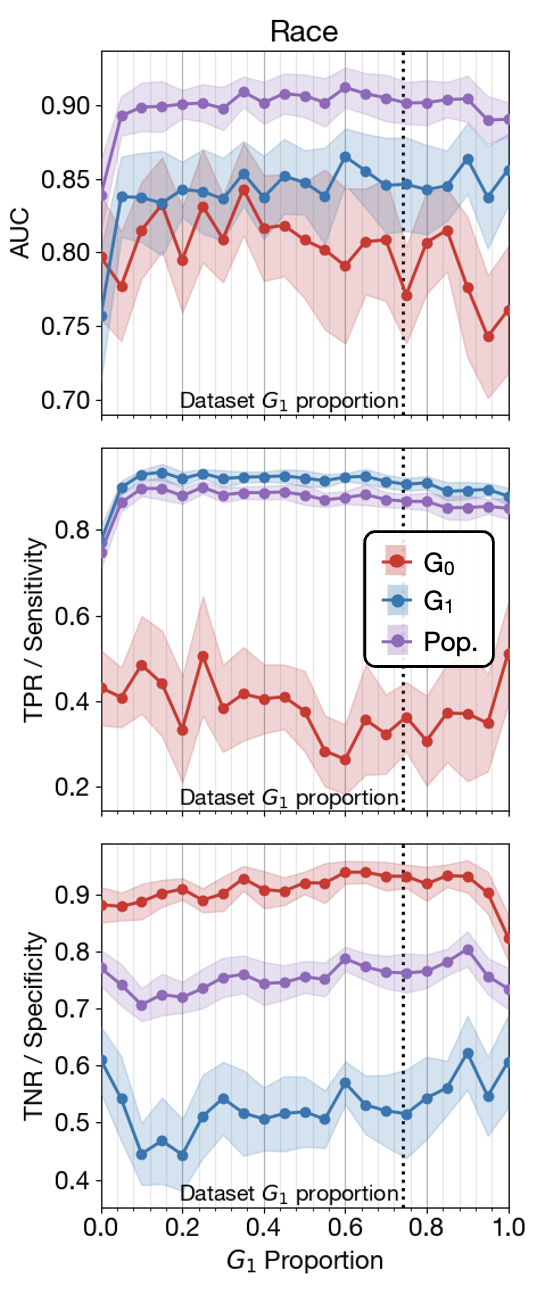}
    \caption{Population (purple) and subgroup (red and blue) AUCs, TPRs, and TNRs for gradient boosted classifiers in the Community Crime dataset.
    Circles and shaded regions indicate quantile means and 95\% CIs for performance of representativeness-based samplers with varying $G_1$ proportions.
    Each row indicates a different classifier performance measure.}
    \label{fig:crm_fairness}
\end{figure*}

We present analyses of population and group-wise accuracy, true positive rates, and true negative rates of classifiers trained on the remaining datasets known to have unfairness that are not included in the main body (Law School, Community Crime).
The methodology and presentation of these results parallels main body figure \ref{fig:adl_fairness}.
There is significant TPR and TNR unfairness for groups determined by race in both datasets (Figs. \ref{fig:law_non_arm_fairness} and \ref{fig:crm_fairness}).

\subsubsection{Complexity Analysis} \label{sup:complexity_analysis}

\begin{figure*}[tbh]
    \centering
    \includegraphics[width=0.95\linewidth]{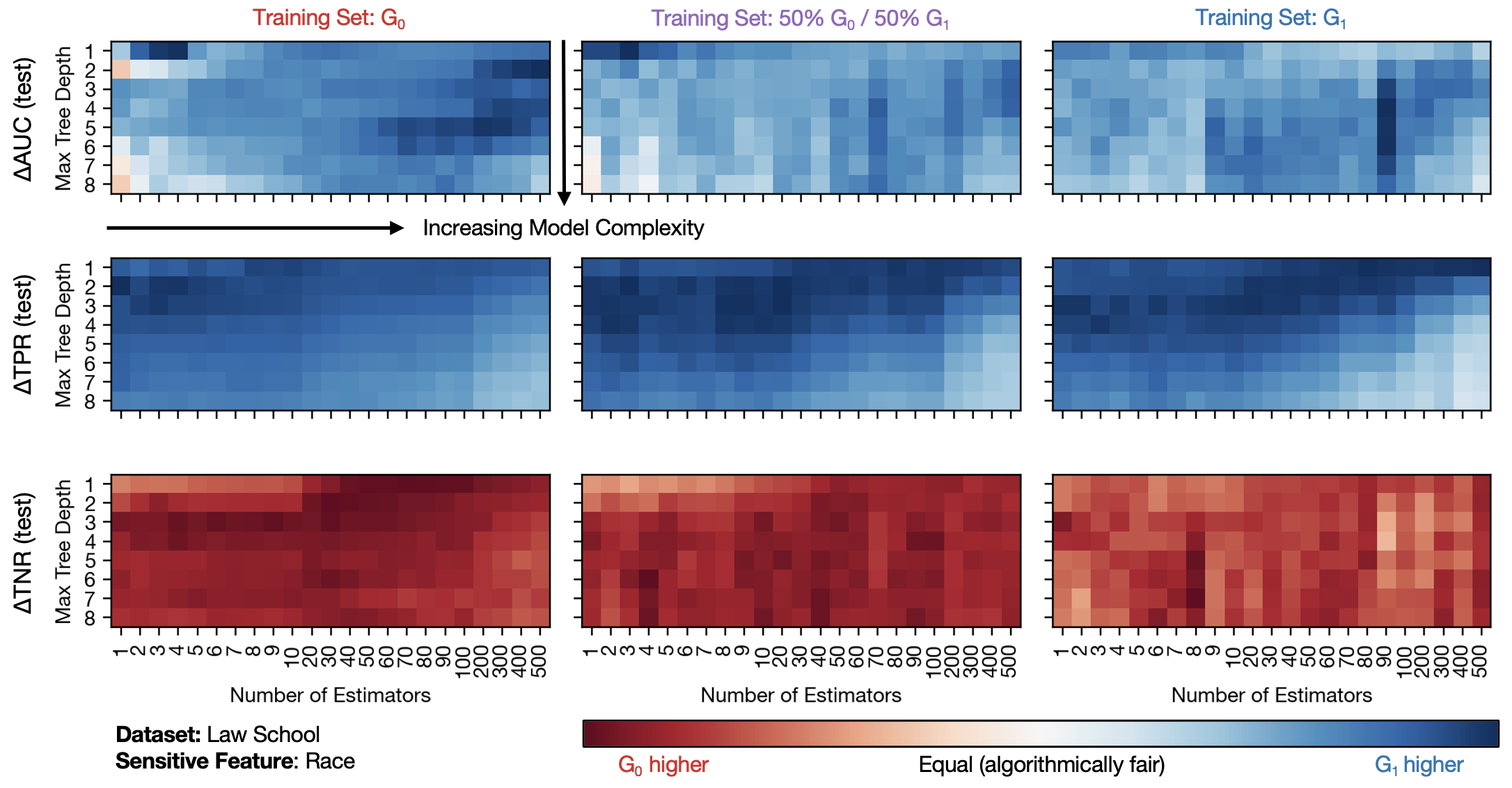}
    \caption{GBC unfairness for the Law School dataset treating race as the sensitive feature of interest.
    Darker red and blue colors indicate disparate performance favoring group $G_0$ and $G_1$, respectively, while paler colors indicate measure parity (fairness).
    Within each subfigure, rows represent maximum individual tree depths and columns indicate numbers of estimation steps.}
    \label{fig:law_complexity_heatmap_race}
\end{figure*}

\begin{figure*}[tbh]
    \centering
    \includegraphics[width=0.95\linewidth]{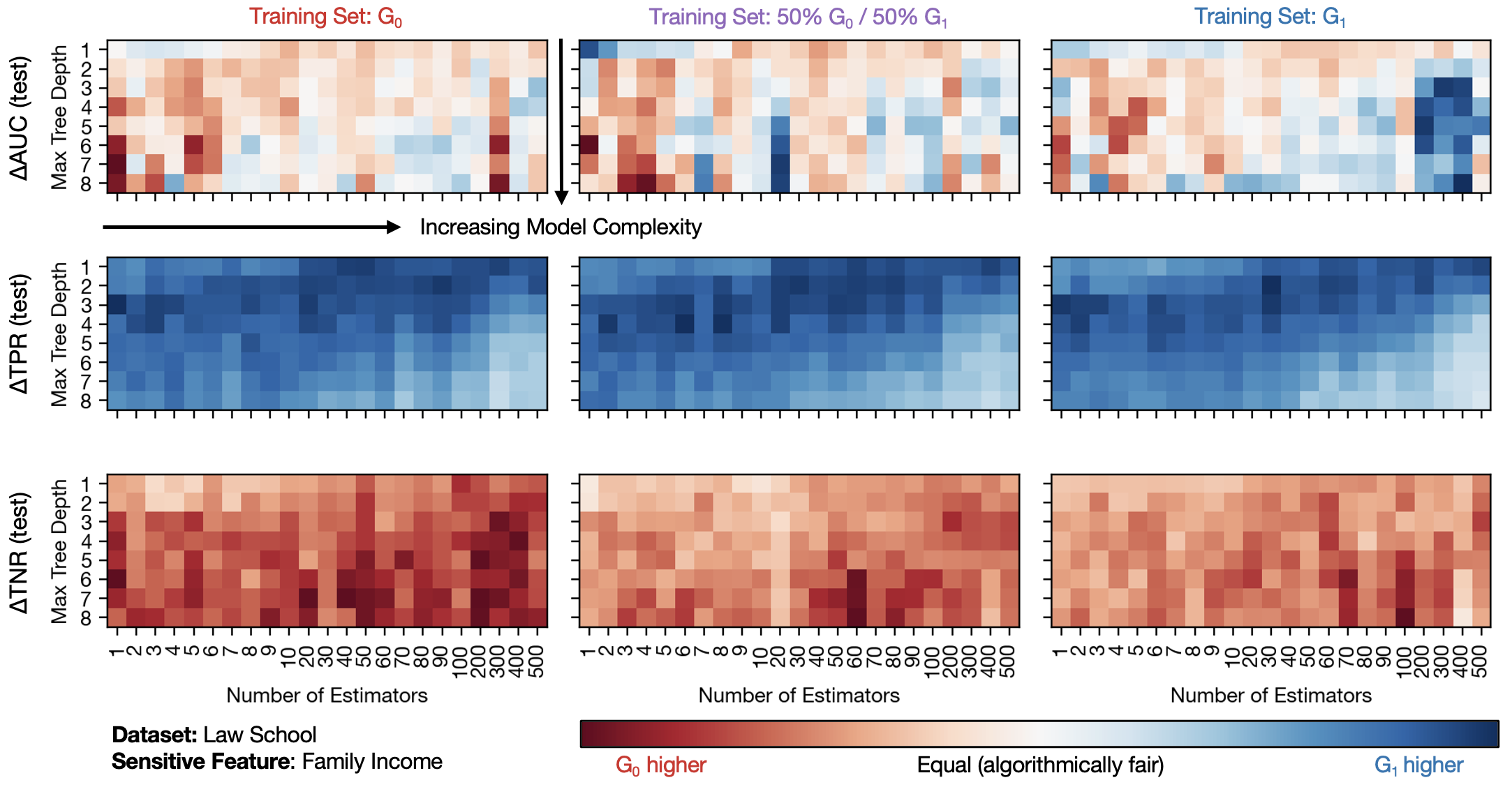}
    \caption{GBC unfairness for the Law School dataset treating family income as the sensitive feature of interest.
    Darker red and blue colors indicate disparate performance favoring group $G_0$ and $G_1$, respectively, while paler colors indicate measure parity (fairness).
    Within each subfigure, rows represent maximum individual tree depths and columns indicate numbers of estimation steps.}
    \label{fig:law_complexity_heatmap_fam_inc}
\end{figure*}

\begin{figure*}[tbh]
    \centering
    \includegraphics[width=0.95\linewidth]{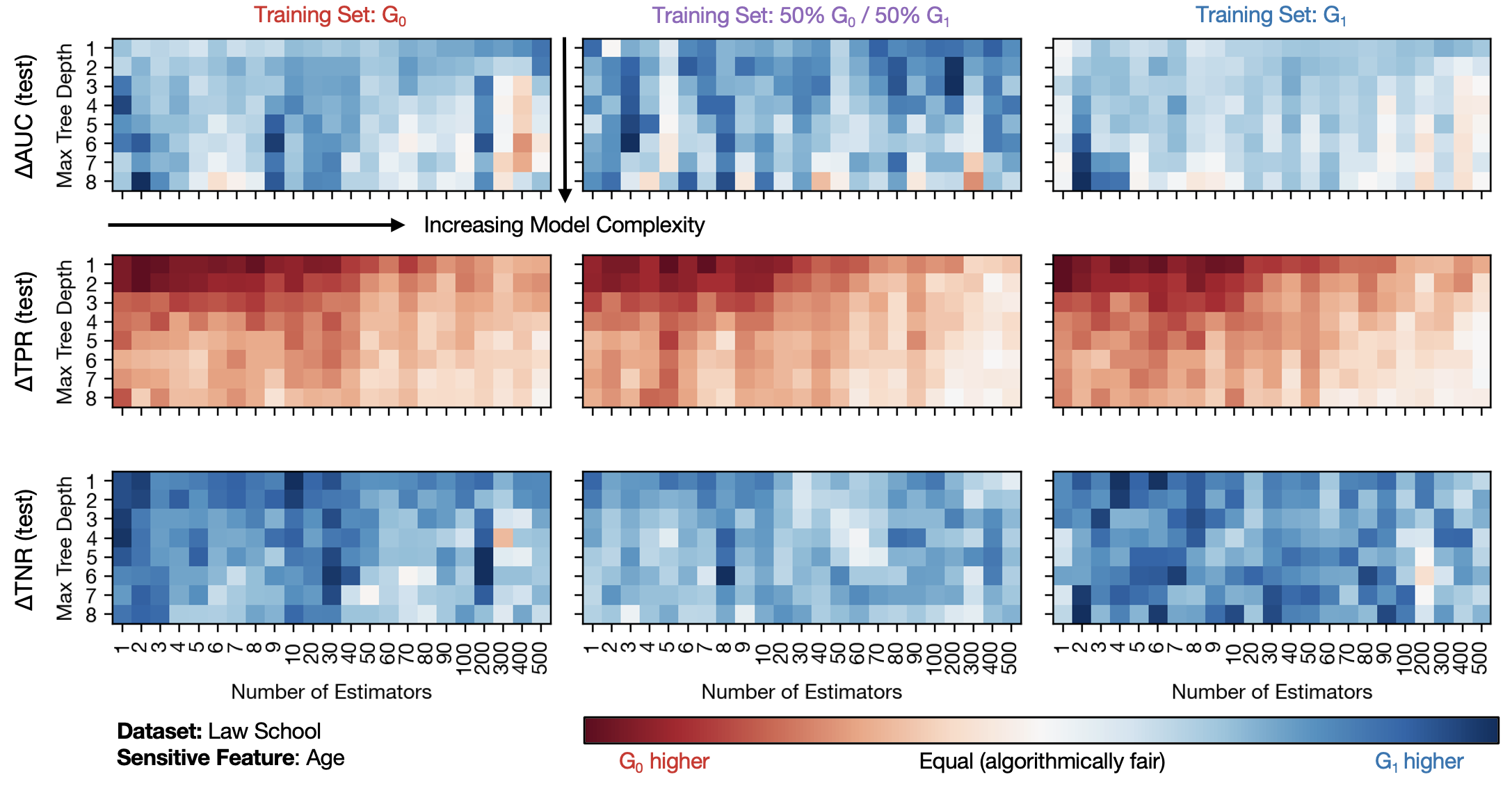}
    \caption{GBC unfairness for the Law School dataset treating age as the sensitive feature of interest.
    Darker red and blue colors indicate disparate performance favoring group $G_0$ and $G_1$, respectively, while paler colors indicate measure parity (fairness).
    Within each subfigure, rows represent maximum individual tree depths and columns indicate numbers of estimation steps.}
    \label{fig:law_complexity_heatmap_age}
\end{figure*}

\begin{figure*}[tbh]
    \centering
    \includegraphics[width=0.95\linewidth]{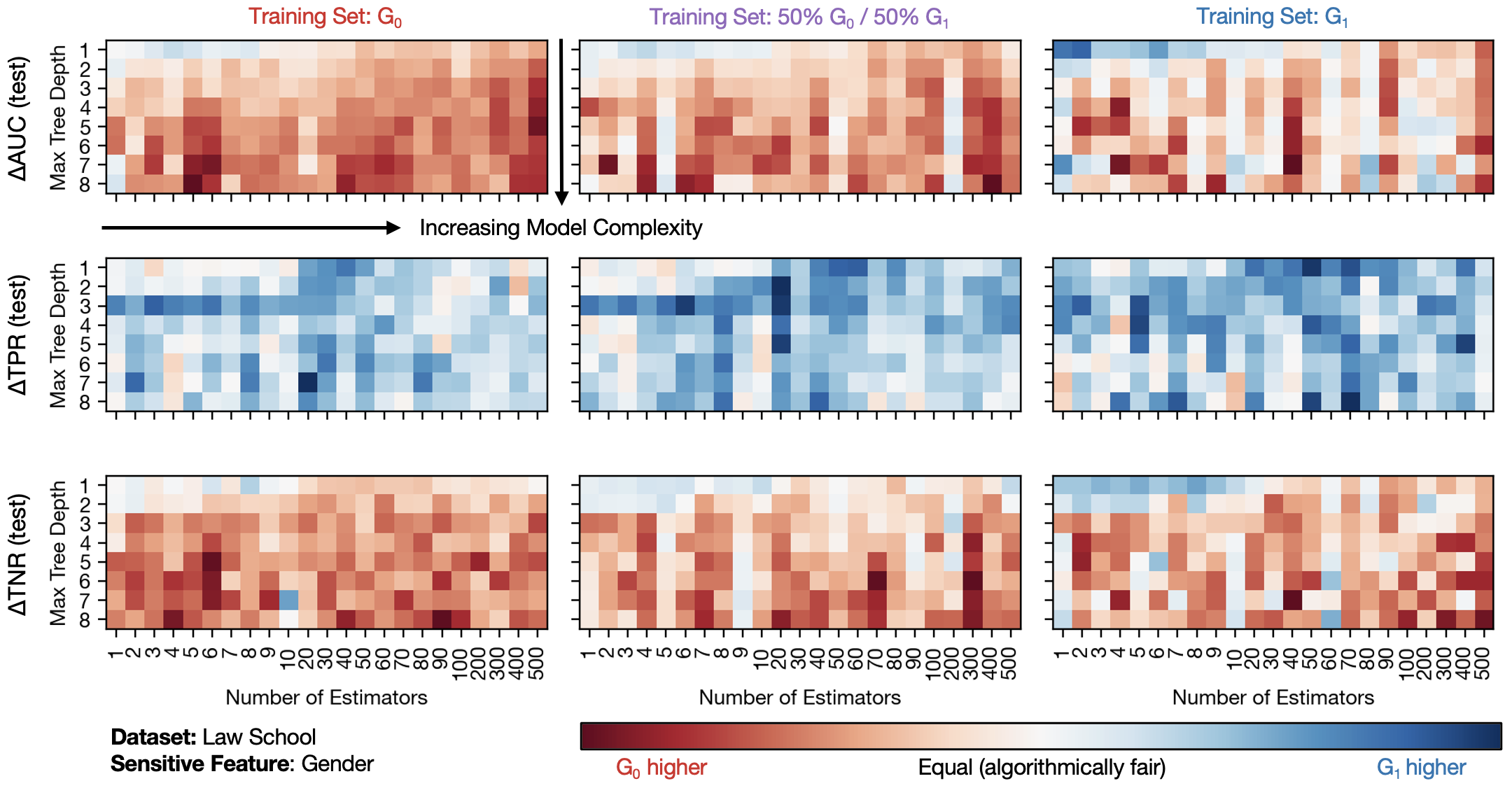}
    \caption{GBC unfairness for the Law School dataset treating race as the sensitive feature of interest.
    Darker red and blue colors indicate disparate performance favoring group $G_0$ and $G_1$, respectively, while paler colors indicate measure parity (fairness).
    Within each subfigure, rows represent maximum individual tree depths and columns indicate numbers of estimation steps.}
    \label{fig:law_complexity_heatmap_gender}
\end{figure*}

\begin{figure*}[tbh]
    \centering
    \includegraphics[width=0.95\linewidth]{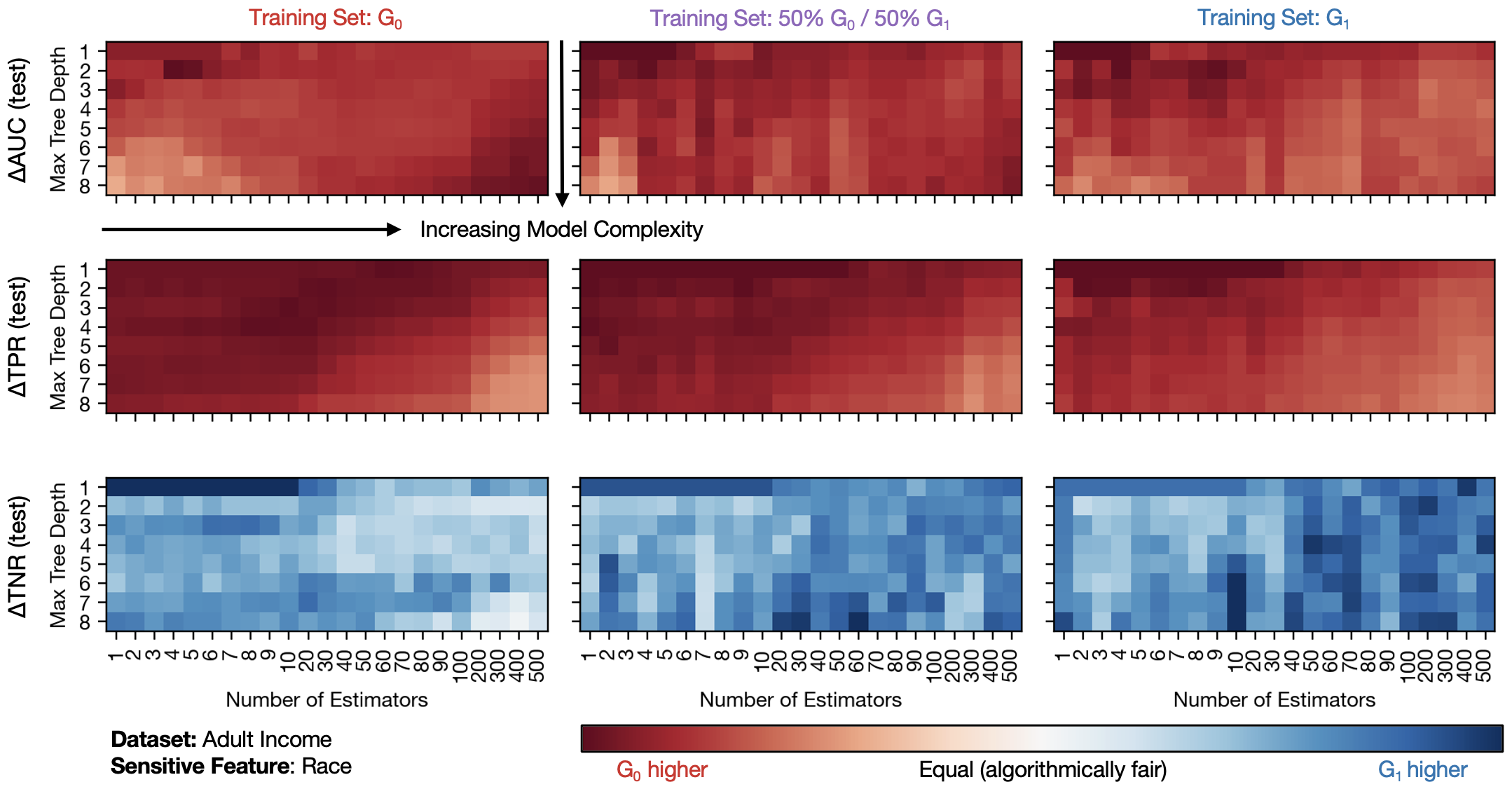}
    \caption{GBC unfairness for the Adult Income dataset treating race as the sensitive feature of interest.
    Darker red and blue colors indicate disparate performance favoring group $G_0$ and $G_1$, respectively, while paler colors indicate measure parity (fairness).
    Within each subfigure, rows represent maximum individual tree depths and columns indicate numbers of estimation steps.}
    \label{fig:adl_complexity_heatmap_race}
\end{figure*}

\begin{figure*}[tbh]
    \centering
    \includegraphics[width=0.95\linewidth]{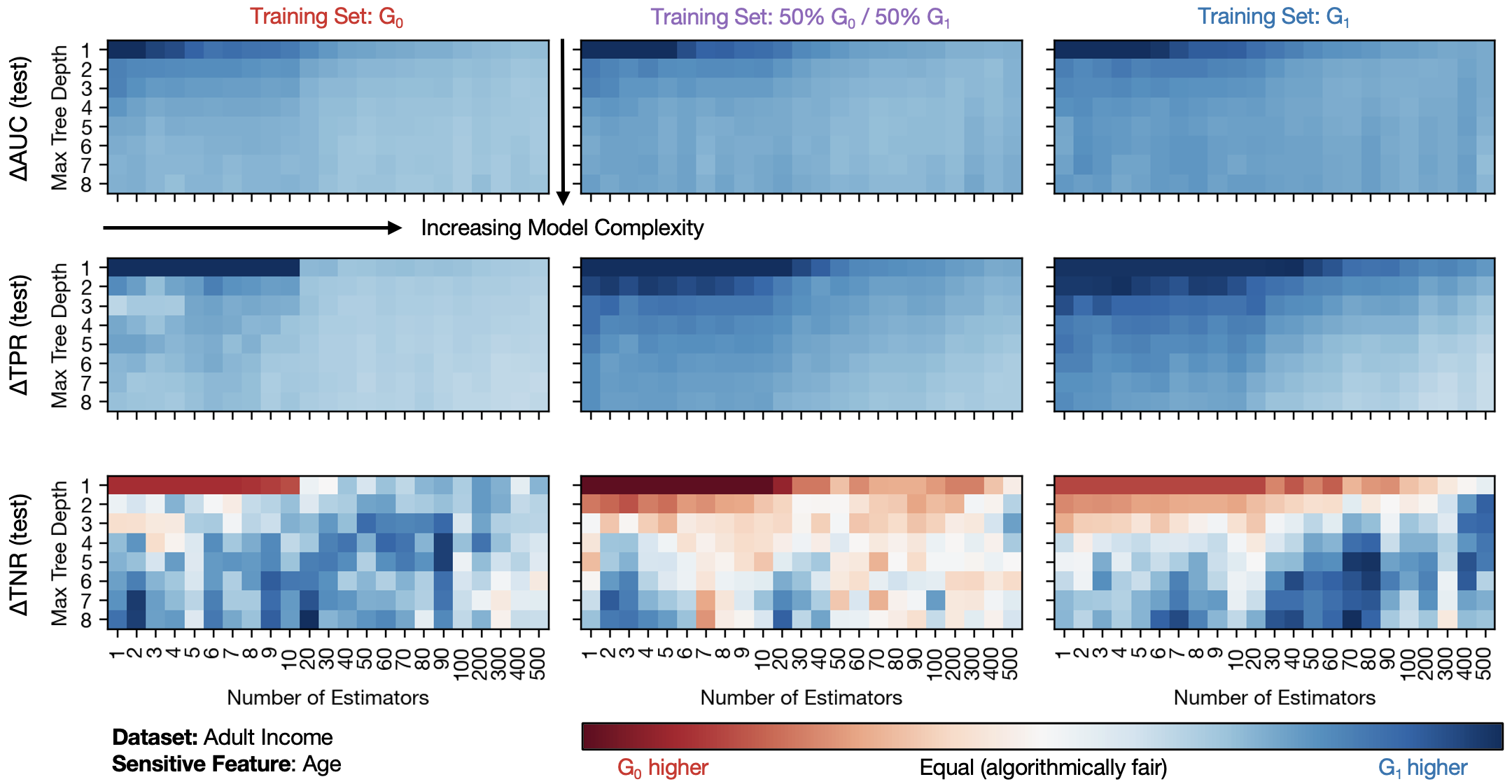}
    \caption{GBC unfairness for the Adult Income dataset treating age as the sensitive feature of interest.
    Darker red and blue colors indicate disparate performance favoring group $G_0$ and $G_1$, respectively, while paler colors indicate measure parity (fairness).
    Within each subfigure, rows represent maximum individual tree depths and columns indicate numbers of estimation steps.}
    \label{fig:adl_complexity_heatmap_age}
\end{figure*}

\begin{figure*}[tbh]
    \centering
    \includegraphics[width=0.95\linewidth]{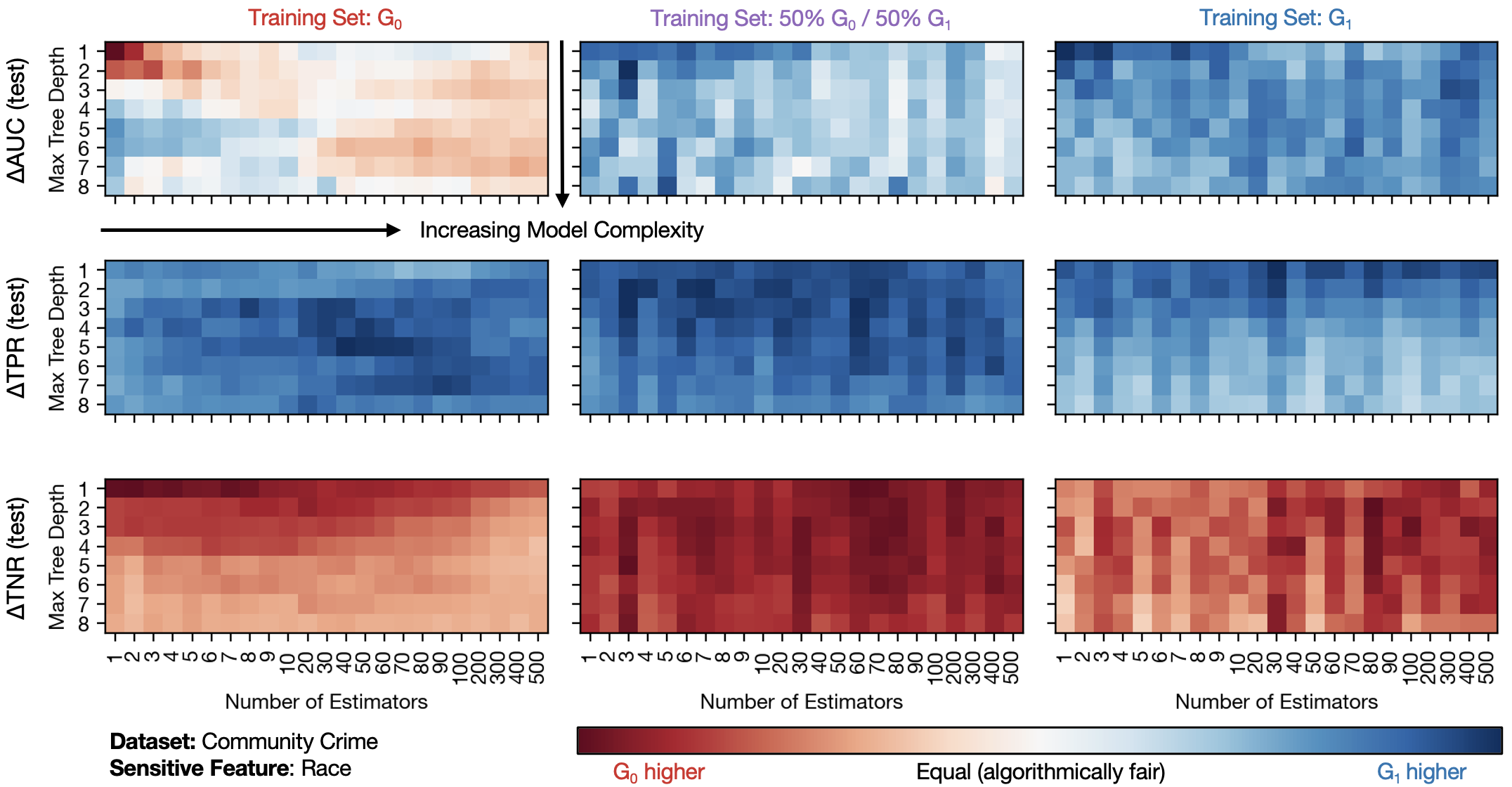}
    \caption{GBC unfairness for the Community Crime dataset treating race as the sensitive feature of interest.
    Darker red and blue colors indicate disparate performance favoring group $G_0$ and $G_1$, respectively, while paler colors indicate measure parity (fairness).
    Within each subfigure, rows represent maximum individual tree depths and columns indicate numbers of estimation steps.}
    \label{fig:crm_complexity_heatmap_race}
\end{figure*}

We include results for complexity analyses of all sensitive features on all datasets known to have unfairness (Law School, Adult Income, Community Crime).
The methodology and presentations of these results parallels main body figure \ref{fig:adl_gender_complexity}.
In general, increased model complexity yields better AUC, TPR, and/or TNR parity --- thus, fairer models (Figs. \ref{fig:law_complexity_heatmap_race}, \ref{fig:law_complexity_heatmap_fam_inc}, \ref{fig:law_complexity_heatmap_age}, \ref{fig:adl_complexity_heatmap_race}, and \ref{fig:adl_complexity_heatmap_age}).
However, there are a couple cases where increasing model complexity does not significantly improve fairness (Figs. \ref{fig:law_complexity_heatmap_gender} and \ref{fig:crm_complexity_heatmap_race}).
Nevertheless, it does not appear that increasing model complexity \textit{harms} fairness, making it at least a potentially beneficial intervention from a fairness perspective.

\subsubsection{Performance and Complexity Analysis} \label{sup:perf_complexity_analysis}
\begin{figure}[tbh]
    \centering
    \includegraphics[width=1\linewidth]{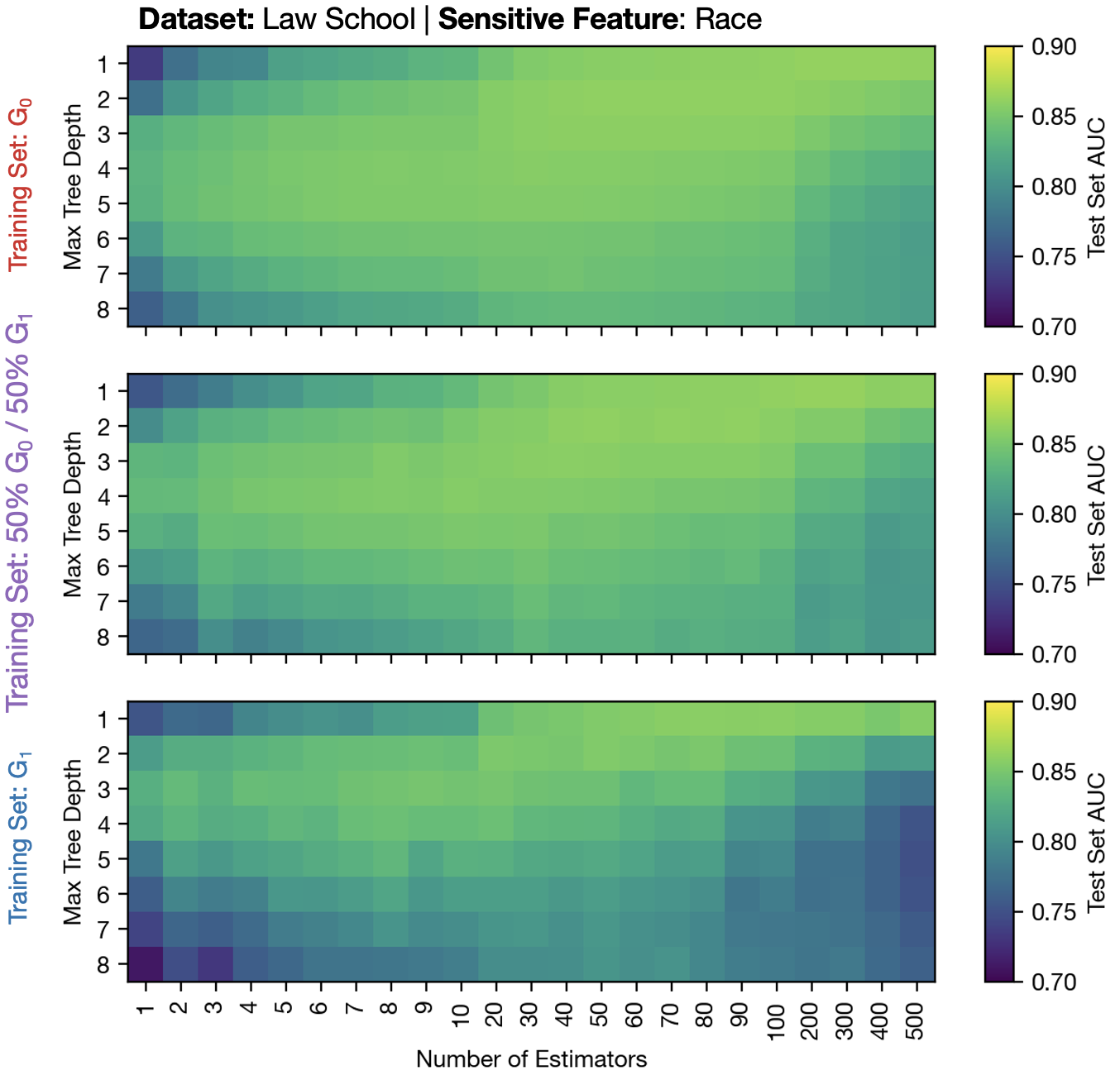}
    \caption{Heatmaps show total test set AUCs across different training subgroups in the Law School dataset treating race as the sensitive feature of interest.
    }
    \label{fig:law_race_auc}
\end{figure}

\begin{figure}[tbh]
    \centering
    \includegraphics[width=1\linewidth]{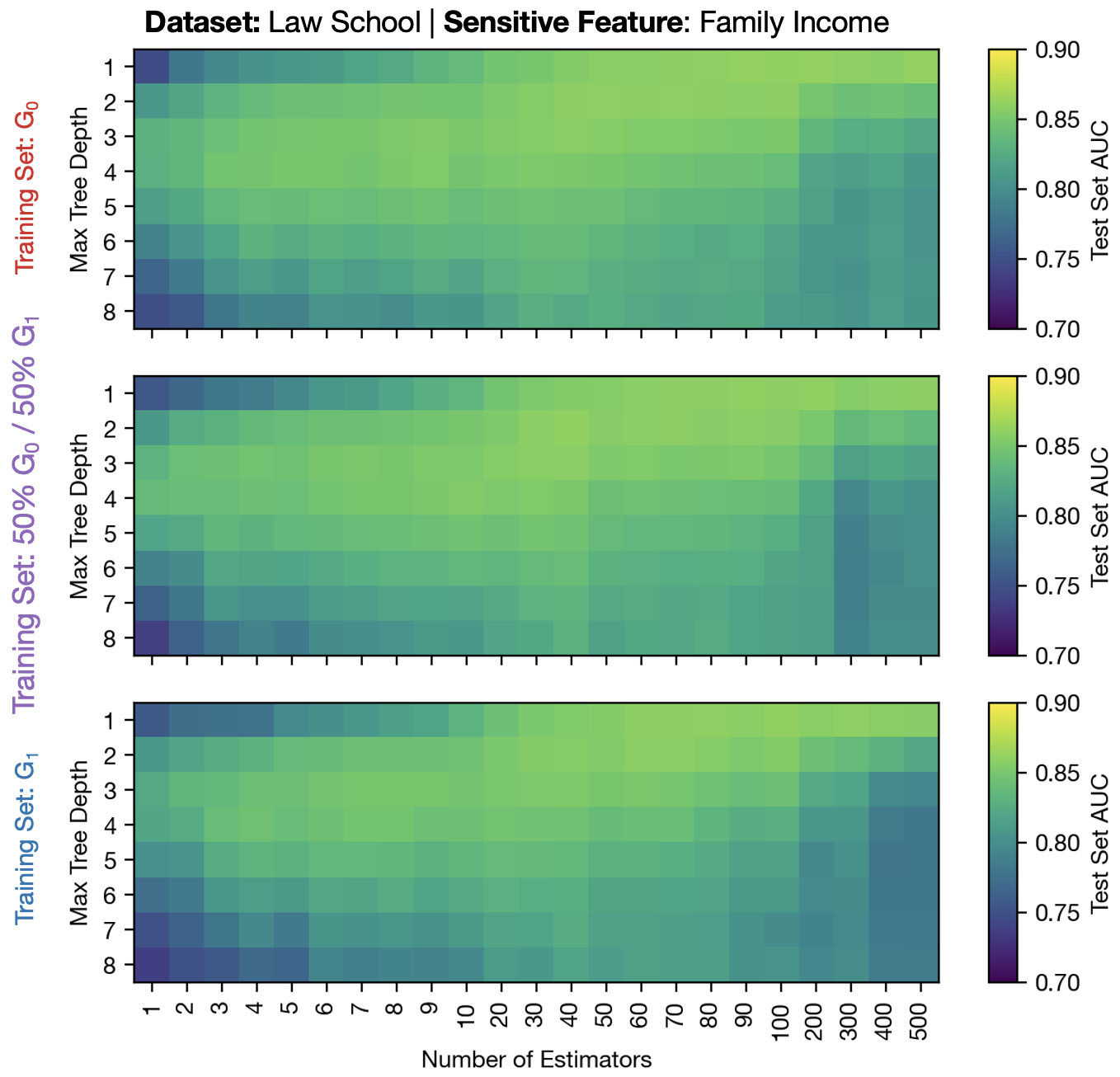}
    \caption{Heatmaps show total test set AUCs across different training subgroups in the Law School dataset treating family income as the sensitive feature of interest.
    }
    \label{fig:law_fam_inc_auc}
\end{figure}

\begin{figure}[tbh]
    \centering
    \includegraphics[width=1\linewidth]{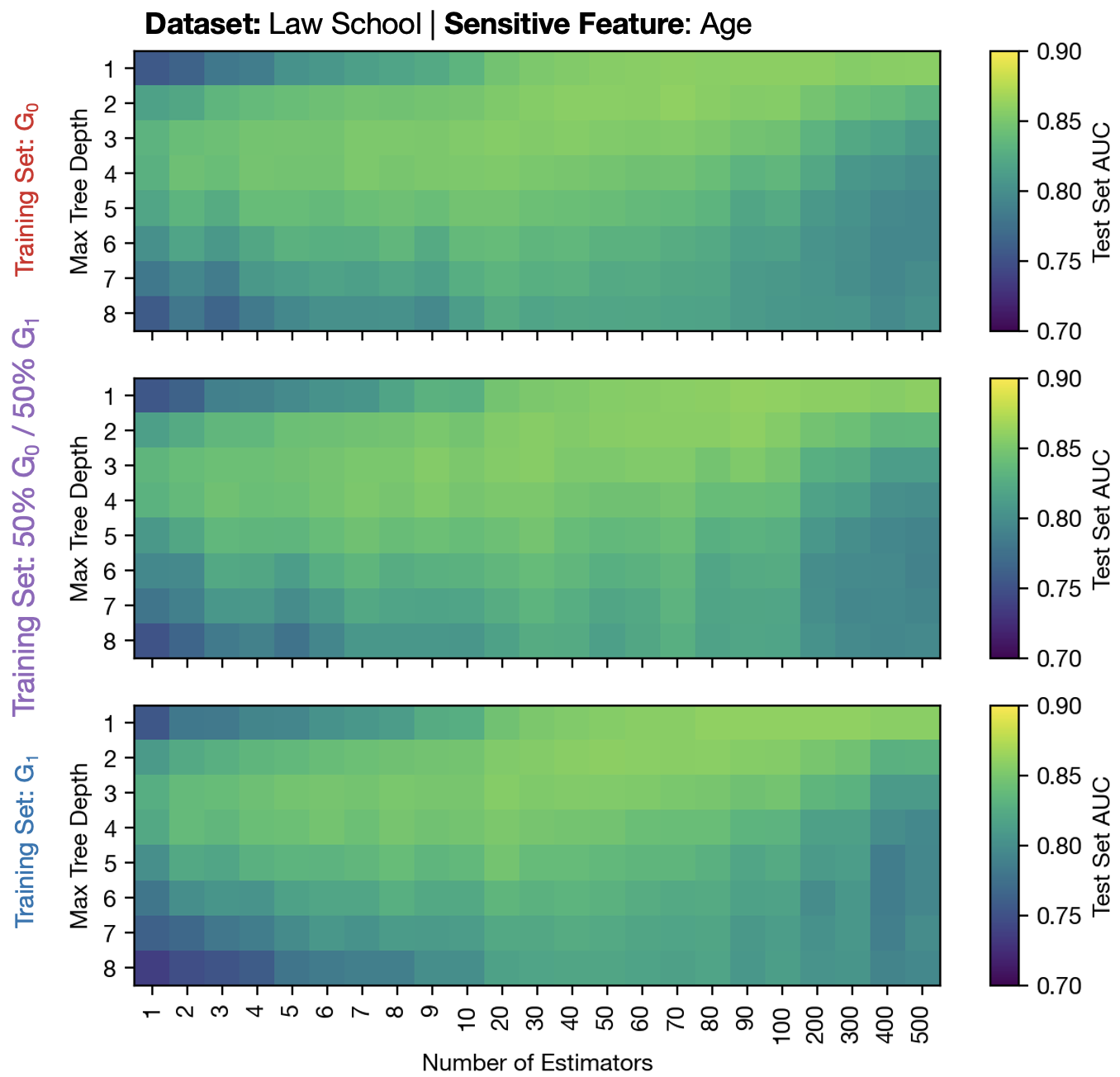}
    \caption{Heatmaps show total test set AUCs across different training subgroups in the Law School dataset treating age as the sensitive feature of interest.
    }
    \label{fig:law_age_auc}
\end{figure}

\begin{figure}[tbh]
    \centering
    \includegraphics[width=1\linewidth]{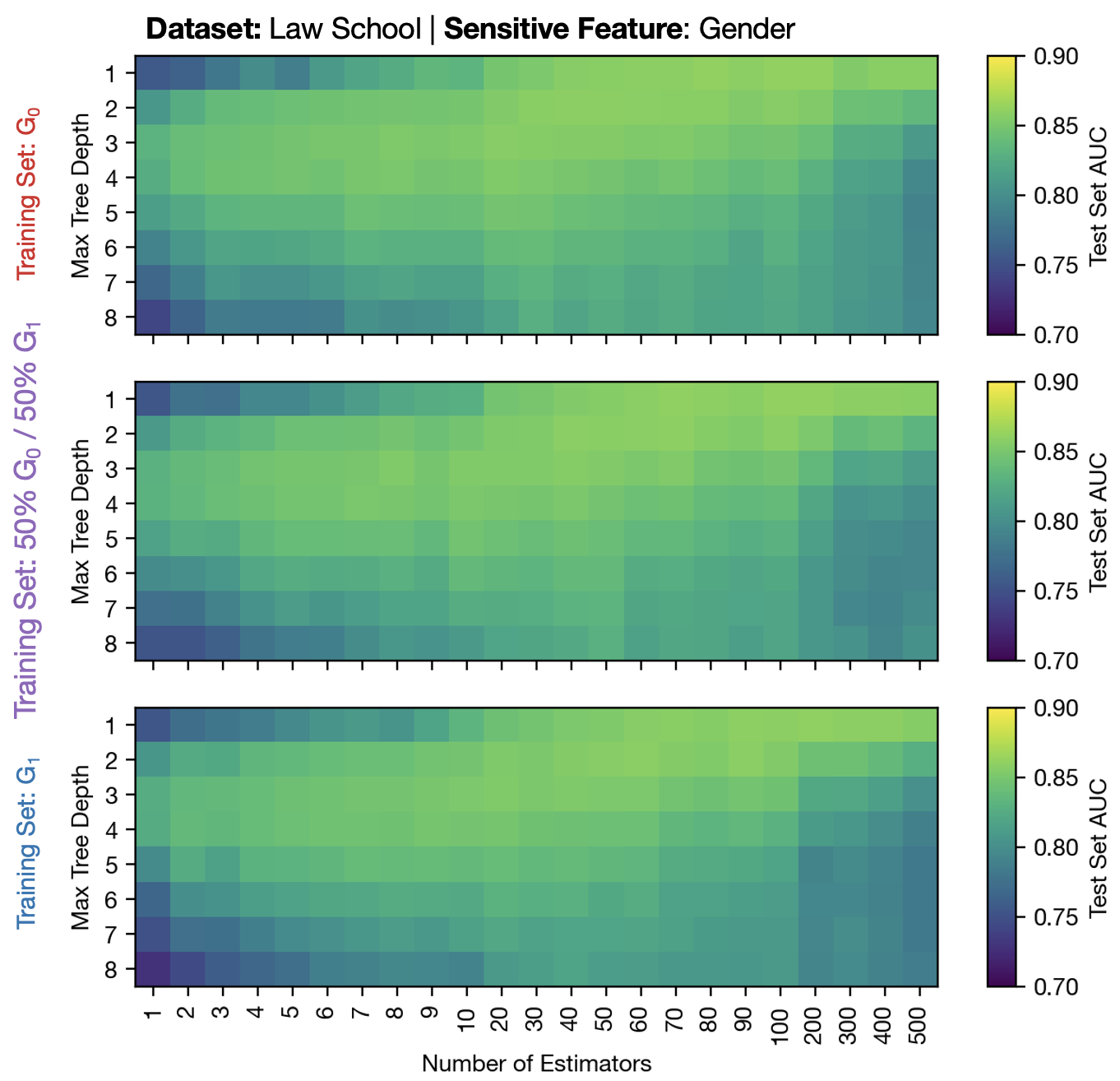}
    \caption{Heatmaps show total test set AUCs across different training subgroups in the Law School dataset treating gender as the sensitive feature of interest.
    }
    \label{fig:law_gender_auc}
\end{figure}

\begin{figure}[tbh]
    \centering
    \includegraphics[width=1\linewidth]{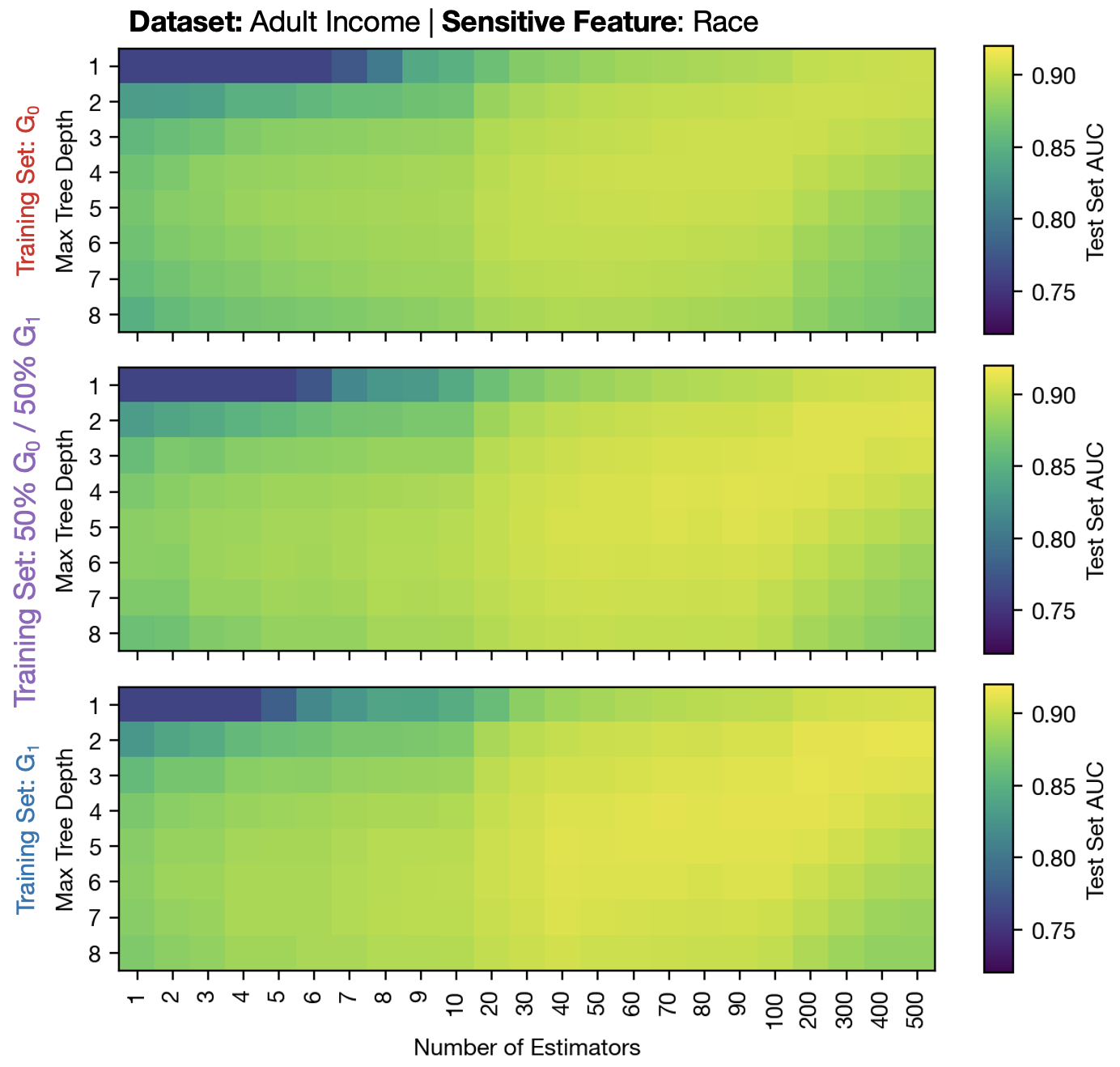}
    \caption{Heatmaps show total test set AUCs across different training subgroups in the Adult Income dataset treating race as the sensitive feature of interest.
    }
    \label{fig:adl_race_auc}
\end{figure}

\begin{figure}[tbh]
    \centering
    \includegraphics[width=1\linewidth]{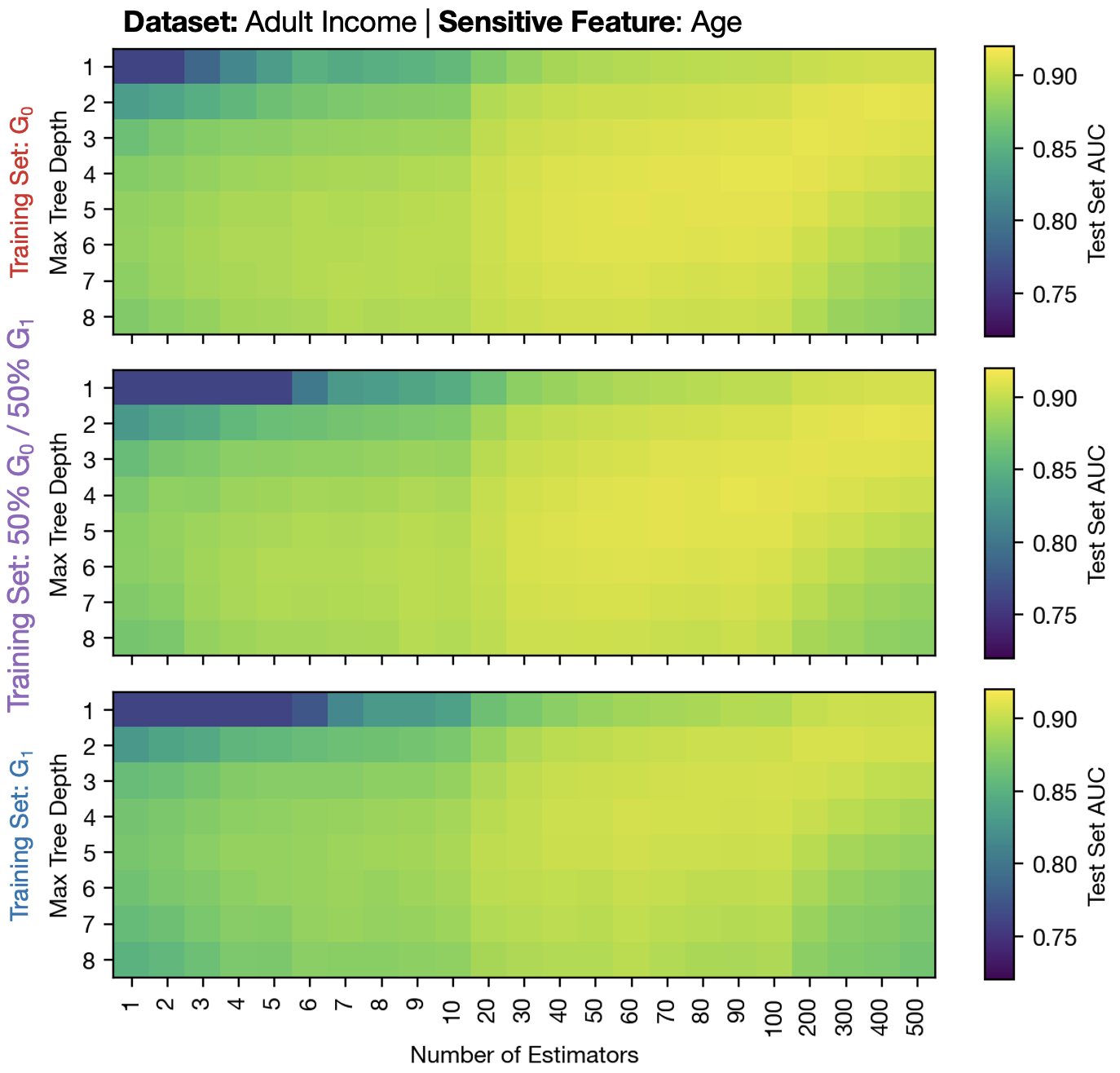}
    \caption{Heatmaps show total test set AUCs across different training subgroups in the Adult Income dataset treating age as the sensitive feature of interest.
    }
    \label{fig:adl_age_auc}
\end{figure}

\begin{figure}[tbh]
    \centering
    \includegraphics[width=1\linewidth]{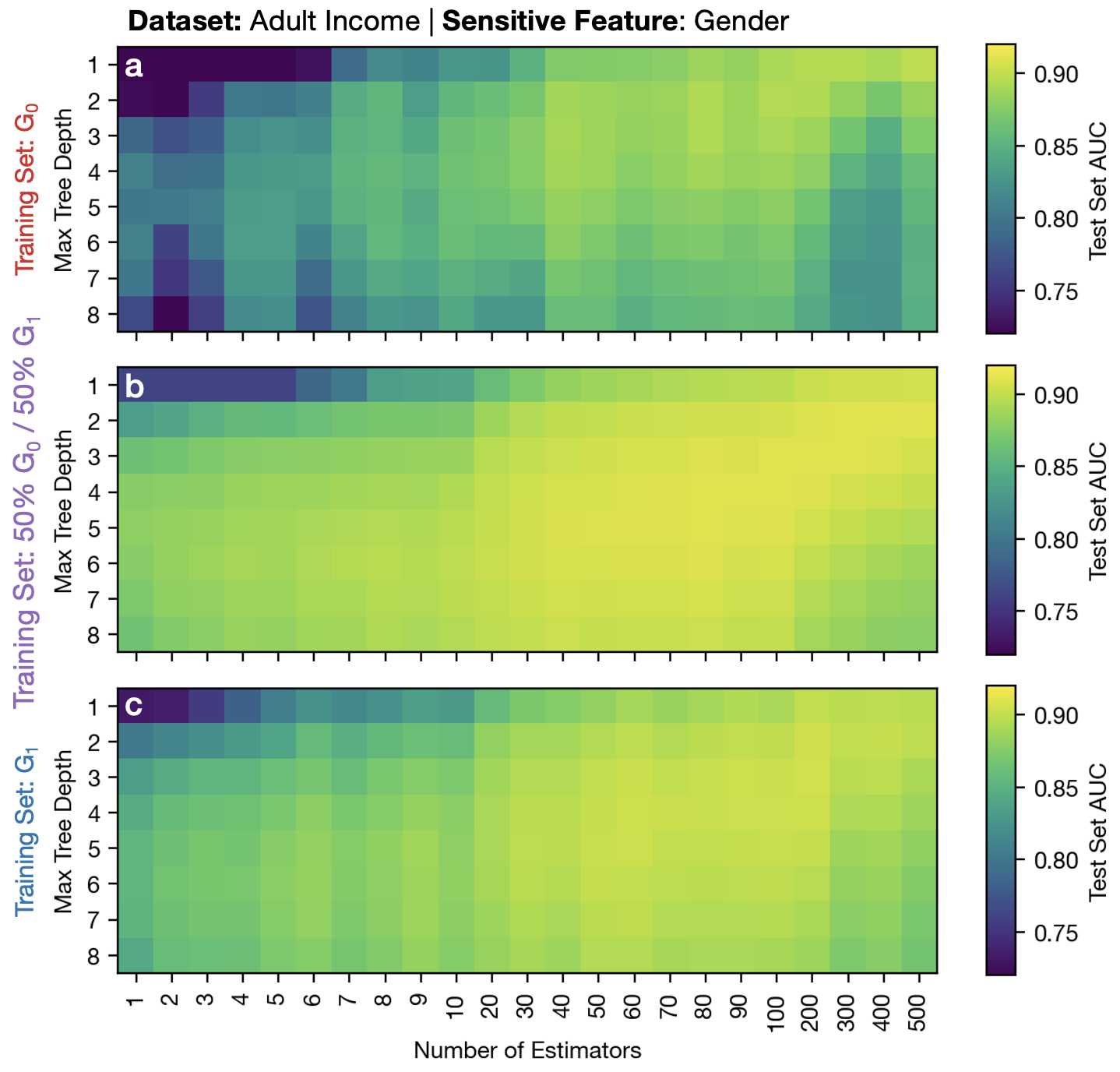}
    \caption{Heatmaps show total test set AUCs across different training subgroups in the Adult Income dataset treating gender as the sensitive feature of interest.
    }
    \label{fig:adl_gender_auc}
\end{figure}

\begin{figure}[tbh]
    \centering
    \includegraphics[width=1\linewidth]{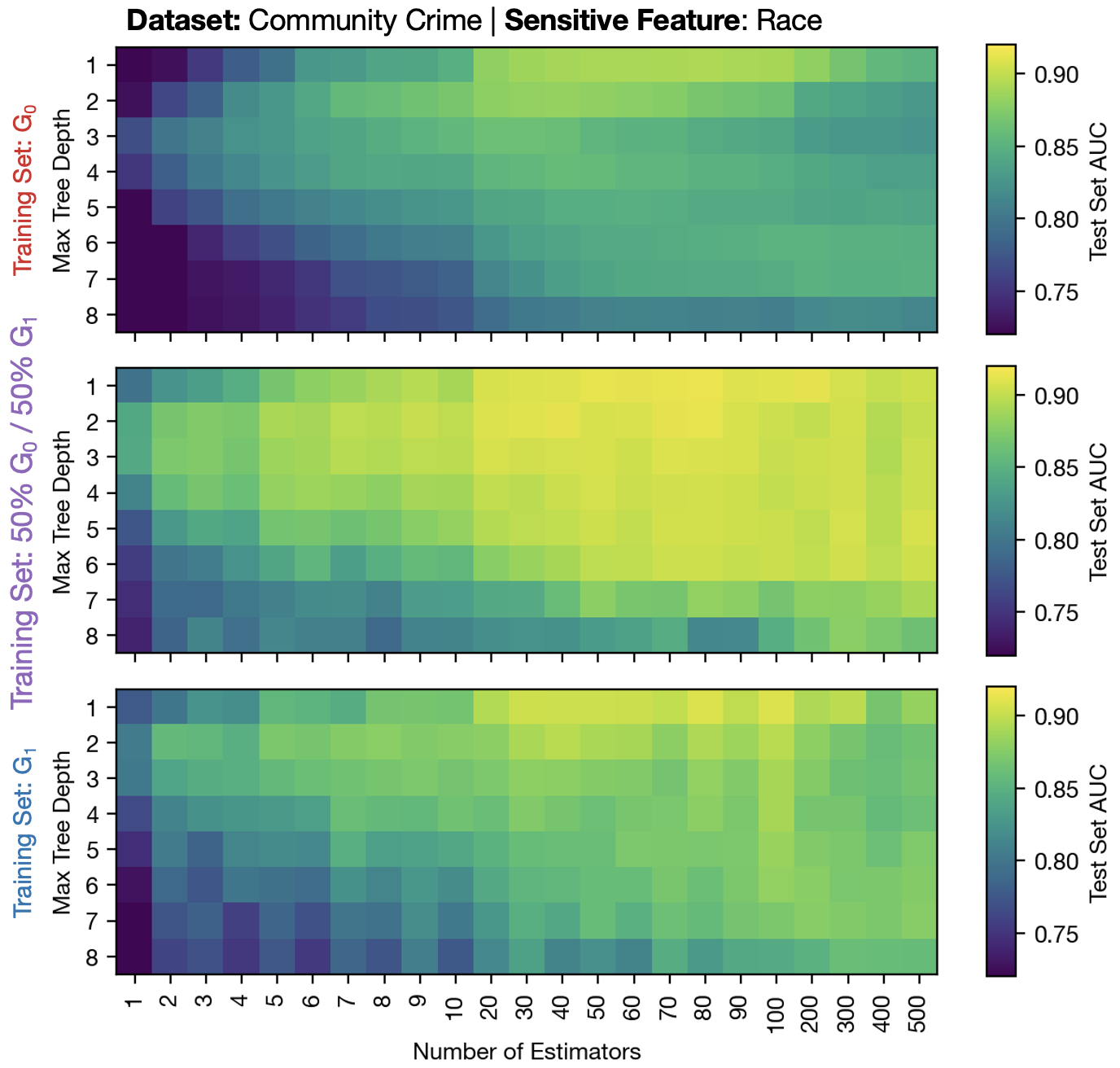}
    \caption{Heatmaps show total test set AUCs across different training subgroups in the Community Crime dataset treating race as the sensitive feature of interest.
    }
    \label{fig:crm_race_auc}
\end{figure}

As discussed in the main body, improvements in algorithmic fairness can often come at the cost of overall classifier performance.
To analyze whether any fairness gains we see from increased model complexity harm classifier performance, we show the overall test set AUC of models with varying complexity.
Each figure in this section parallels a figure in appendix \ref{sup:complexity_analysis} or main body figure \ref{fig:adl_gender_complexity}.
In general, overall classifier AUC does not substantially degrade with increasing model complexity.
The highest complexity levels (estimators $\geq 200$ and depth $\geq 5$) sometimes show moderate degradation in performance (Fig. \ref{fig:law_race_auc}).
However, substantial fairness gains can be realized at lower complexity levels (matched Fig. \ref{fig:law_complexity_heatmap_race}).

\end{document}